\documentclass{article}

\usepackage{wrapfig}
\usepackage{mdframed}
\usepackage{amsmath, amsthm}
\usepackage{cancel}
\usepackage{enumitem}
 \usepackage[preprint]{neurips_2025}
 \usepackage{comment}
 \usepackage{framed}
 \newtheorem{lemma}{Lemma}
 \newtheorem{remark}{Remark}
 \newtheorem{definition}{Definition}
 \newmdtheoremenv[
    tikzsetting={rounded corners=5pt},
    backgroundcolor=gray!10,
    linecolor=white,
    innertopmargin=6pt,
    innerbottommargin=2pt,
    skipabove=5pt,
    skipbelow=5pt
]{theorem}{Theorem}
 \newtheorem{corollary}{Corollary}
 \newtheorem{assumption}{Assumption}
 \newmdtheoremenv[
    tikzsetting={rounded corners=5pt},
    backgroundcolor=gray!10,
    linecolor=white,
    innertopmargin=6pt,
    innerbottommargin=2pt,
    skipabove=5pt,
    skipbelow=5pt
]{proposition}{Proposition}
 \newmdtheoremenv[
    tikzsetting={rounded corners=5pt},
    backgroundcolor=gray!10,
    linecolor=white,
    innertopmargin=6pt,
    innerbottommargin=2pt,
    skipabove=5pt,
    skipbelow=5pt
]{conjecture}{Conjecture}
\DeclareMathOperator*{\argmin}{argmin}

\usepackage[utf8]{inputenc} % allow utf-8 input
\usepackage[T1]{fontenc}    % use 8-bit T1 fonts
\usepackage{hyperref}       % hyperlinks
\hypersetup{
    colorlinks,
    linkcolor={red!50!black},
    citecolor={blue!50!black},
    urlcolor={blue!80!black}
}
\usepackage{url}            % simple URL typesetting
\usepackage{booktabs}       % professional-quality tables
\usepackage{amsfonts}       % blackboard math symbols
\usepackage{nicefrac}       % compact symbols for 1/2, etc.
\usepackage{microtype}      % microtypography
\usepackage{xcolor}         % colors
\newcommand{\hbX}{\bar{\mathsf{X}}}

\title{Diffusion Processes on Implicit Manifolds}

\author{%
  Victor Kawasaki-Borruat\thanks{Corresponding author: \texttt{victor.borruat@epfl.ch}}  \\
  Signal Processing Laboratory 2\\ 
  EPFL \\
  \And Clara Grotehans \\
  Institute of Artificial Intelligence \\
  Medical University of Vienna \AND Pierre Vandergheynst  \\
  Signal Processing Laboratory 2\\ 
  EPFL \And Adam Gosztolai\\
  Institute of Artificial Intelligence \\
  Medical University of Vienna
}

\usepackage{todonotes}

\newcommand{\rwgl}{\mathsf{L}}
\newcommand{\discrop}{\mathsf{P}}
\newcommand{\interp}{\mathsf{I}}

\newcommand{\unif}{\textnormal{dvol}_g}
\newcommand{\dd}{\textnormal{d}}

\begin{document}

\maketitle

\begin{abstract}
High-dimensional data are often assumed to lie on lower-dimensional manifolds. We study how to construct diffusion processes on this data manifold using only point cloud samples and without access to charts, projections, or other geometric primitives.
Here, we introduce Implicit Manifold-valued Diffusions (IMDs), a data-driven mathematical formalism for defining stochastic differential equations in the original high-dimensional space that describe drifting Brownian particles evolving intrinsically on the underlying manifold. 
Our construction hinges on approximating the corresponding infinitesimal generator of the diffusion process using a proximity graph over the data and using the carré-du-champ of the generator, which encodes the local tangent spaces of the manifold and lifts the intrinsic process into ambient coordinates. 
We show that as the number of samples grows, our discrete diffusion process converges in law on the space of probability paths to its smooth manifold counterpart.
We further present an Euler-Maruyama scheme for the numerical integration of IMDs. 
We validate our framework using numerical experiments on synthetic manifolds and the MNIST data manifold, showing that IMDs remain confined over the manifold and enable its guided exploration.
Our work provides the mathematical foundation and practical implementations of diffusion processes on data manifolds, opening new avenues for manifold-aware sampling, exploration, and generative modeling.
\end{abstract}

\definecolor{darkblue}{RGB}{0,0,139}

\section{Introduction}\label{sec:intro}

% ---------------- New Intro ---------------------------------
% paragraph 1: explain why stochastic dynamics are needed
Stochastic differential equations (SDEs) model continuous-time dynamics under deterministic drift and random perturbations \citep{oksendal2003stochastic}. In machine learning, they underpin a broad range of methods, including Langevin sampling \citep{bharath2025sampling, neal2011mcmc}, Bayesian inference \citep{welling2011bayesian, mandt2017stochastic}, stochastic optimization \citep{absil2008optimization, xu2018global, bonnabel2013stochastic, boumal2023intromanifolds}, and score-based generative modeling \citep{song2021scorebasedgenerativemodelingstochastic}.
While most existing methods are formulated in Euclidean space $\mathbb R^n$, many datasets are better modeled as concentrating near lower-dimensional manifolds of intrinsic dimension $d \ll n$, known as the \emph{manifold hypothesis} \citep{fefferman2016testing}. Under these conditions, a diffusion process $(Z_t)_{t\geq 0}$ over a manifold $\mathcal{M}$ may be described by the SDE 
\begin{equation}\label{eq:manifold-diffusion}
    d Z_t = b(Z_t)d t + \sigma(Z_t)d B^\mathcal{M}_t,
\end{equation}
where $b(z),\sigma(z)$ are a smooth drift and diffusion fields, and $B^\mathcal{M}_t$ is Brownian motion on $\mathcal{M}$ \citep{hsu2002stochastic, stroock2007multidimensional}. 
Analytically defining or numerically approximating the dynamics of Eq. \eqref{eq:manifold-diffusion} typically assumes that $\mathcal{M}$ is known, which allows explicitly constructing geometric primitives such as charts \citep{bakry2013analysis},   
retraction operators 
\citep{hairer2006geometric, cheng2022RiemannianMCMC, cheng2022theory, debortoli2022riemannianscorebasedgenerativemodelling, zaghen2025riemannianvariationalflowmatching, gao2023data}, exponential maps \citep{malham2008stochastic} or geodesics \citep{chen2023flow}. Yet, in most practical scenarios, $\mathcal{M}$ can be assumed to exist, but its explicit construction based on the above primitives is either intractable or impractical. Instead, one desires to construct the diffusion process in a data-driven manner given a point cloud $X_N = \{x_i\}_{i=1}^N$ in  $ \mathbb R^n$.

%Here, where $F:\mathcal{M}\to \mathbb R^n$ is an isometric embedding, so that the data are available only through their ambient representation. 
Here we introduce a mathematical formalism for constructing, from the point cloud alone, an ambient-space stochastic process $Y_t\in F(\mathcal{M})\subset \mathbb R^n$, where $F:\mathcal{M}\to \mathbb R^n$ is an isometric embedding, whose evolution matches the intrinsic diffusion $Z_t\in\mathcal{M}$ in the large-sample limit. 
\begin{wrapfigure}{r}{0.30\linewidth}
  \centering
  \vspace{15pt}
  \includegraphics[width=\linewidth]{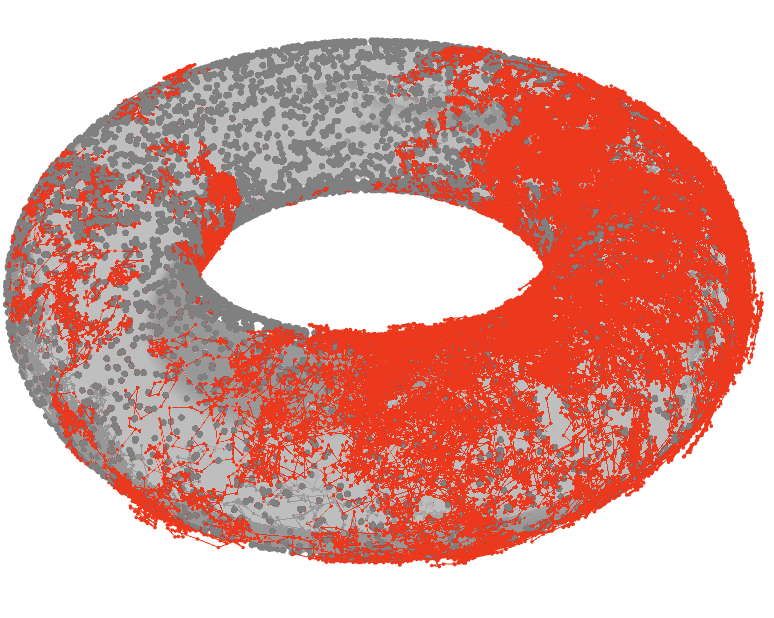}
  \caption{ No prior knowledge of $\mathbb T^2$ beyond the samples (gray) is required to compute the displayed Brownian motion (red).}
  \label{fig:torus-BM}
  \vspace{5pt}
\end{wrapfigure}
To achieve this, we develop an operator-theoretic approach to construct diffusion processes on implicit manifolds. Starting from a proximity graph built from the point cloud $X_N$, we consider the associated random walk graph Laplacian $\rwgl_N$; the discrete generator of a Markov process on the graph. We then show that, in the large-sample limit, $\rwgl_N$ converges to a second-order differential operator $L$, the scaled \emph{infinitesimal generator} of the diffusion $Z_t$ on $\mathcal{M}$ (Theorem \ref{thrm:mosco-conv}).
Moreover, we prove that the constructed process $Y_t$ converges on the space of probability paths to $F(Z_t)$. Theorem \ref{thrm:sampling-guarantees} proves the consistency of the corresponding Euler-Maruyama discretization. Our numerical experiments on synthetic and real-world datasets show that IMDs make for a viable  method to both explore and generate new samples from data manifolds. Figure \ref{fig:torus-BM} depicts a simple visual example of our construction on the torus.
These results provide the first fully data-driven path space-consistent construction of diffusion processes on implicit manifolds, which we dub Implicit Manifold-valued Diffusions (IMDs). 
We summarize our contributions below.
\begin{itemize}
    \item \textbf{Data-driven formulation of manifold  diffusion.}
        We show that stochastic processes on arbitrary data manifolds can be constructed by estimating their infinitesimal generators, without access to charts, projections, or explicit geometric information. 
    \item \textbf{Path-space convergence guarantees.}
        We prove that our data-driven generator converges to the Laplace-Beltrami operator and that the induced stochastic process converges in law on path space $C[0,T]$ to its smooth manifold counterpart in the large-sample limit. 
    \item \textbf{Recovery of intrinsic stochastic dynamics.} We show empirically that IMDs produce smooth on-manifold trajectories between data points and explore directions inaccessible to score-only retraction baselines.
\end{itemize}
We expect that our model will find applications in diverse AI-for-science fields, including computational biology \citep{yim2023se, hollingsworth2018molecular}, robotics \citep{saveriano2023learning}, material science \citep{zaghen2025riemannianvariationalflowmatching}, and neuroscience \citep{perich2025neural, gosztolai25}, where datasets have strong geometric priors yet are often underexplored.
In machine learning, IMDs allow for tangential stochastic exploration of the data manifold, suggesting a complementary perspective to score-based generative modeling under the manifold hypothesis.

\section{Related Works}
\paragraph{Langevin sampling on manifolds}
Stochastic dynamics on Riemannian manifolds are an important tool in geometric sampling \citep{girolami2011riemann, patterson2013stochastic}, providing a foundation for algorithms such as Riemannian Langevin Markov Chain Monte Carlo \citep{cheng2022RiemannianMCMC, cheng2022theory} and Langevin diffusion-based sampling on manifolds \citep{bharath2025sampling}. This perspective is distinct from Riemannian optimization \citep{absil2008optimization, boumal2023intromanifolds}, as the goal is to construct a stochastic process whose invariant law is the target distribution, not to follow a deterministic trajectory. These methods are commonly discretized using geometric Euler-type schemes, also known as geodesic random walks, in which tangent-space increments are mapped back to the manifold via the exponential map $\exp_x$ or a retraction \citep{jorgensen1975central, bharath2025sampling}. They therefore require access to Riemannian gradients and manifold-specific geometric primitives. In contrast, we aim to construct such stochastic dynamics on implicit manifolds without assuming access to such primitives.

\paragraph{Manifold learning} Existing methods in the implicit manifold setting typically attempt to first recover the manifold structure, before approximating the dynamics. One approach involves learning a latent representation  \citep{alberti2024manifold, niu2023intrinsic} before simulating Brownian motion on the latent coordinates. The use of Diffusion Maps \citep{coifman2006diffusion} provides a principled limit from random walks on graphs to manifold diffusions, enabling latent representations of diffusions \citep{giovanis2025generative}, or estimation of diffusion generators \citep{li2023diffusion} for Langevin sampling. %Such use of Diffusion Maps allows for correct sampling of marginals, but not of the entire path.

\paragraph{Diffusion geometry} A recent line of work in data-driven geometry \citep{jones2024diffusiongeometry, jones2026computing} recovers the implicit manifold's Riemannian geometry from diffusion operators \citep{bakry2013analysis} in the ambient Euclidean space. 
The possibility to compute the carré-du-champ (CDC) using Diffusion Maps inspired manifold-aware generative modeling methods using the CDC in flow matching \citep{bamberger2025carreduchampflow, bamberger2026riemannian}. This method, like IMDs, relies on the computation of the CDC operator to construct the data manifold's geometry. 

\paragraph{Score-based diffusion models} Another avenue involves the connections between score-based generative models and implicit manifolds \citep{pidstrigach2022score, de2022convergence, chakraborty2026generalization, buchanan2025edge, li2025scores}. Kharitenko et al. \citep{kharitenko2025landingscoreriemannianoptimization} propose score-based retractions as approximate
projections onto $F(\mathcal{M})$ to recover manifold operations (such as projections of the Riemannian gradient) in the implicit manifold setting. This allows for good approximations of deterministic Riemannian optimization tasks, but not manifold-valued diffusions.

\paragraph{Geometric limits of graph Laplacians} This is the theoretical backbone of IMDs, and consists of the main proof techniques. Such limits were pioneered in a series of work by Belkin and Niyogi \cite{belkin2003laplacian, belkin2004semi, belkin2008towards}, later refined to pointwise convergence of the random walk graph Laplacian to the Laplace-Beltrami operator \cite{hein2007graph}. More recently, spectral convergence was established by considering continuum limits of point clouds \citep{trillos2015rate, trillos2018variational, garcia2016continuum, trillos2018errorestimatesspectralconvergence}, or metric geometry concepts \citep{burago2015graph}; both of which we use to prove Theorem \ref{thrm:mosco-conv}.

% ------------------------------------------------------------

\section{Background on diffusion processes on manifolds}\label{sec:diffusions-on-manifolds}

Here we provide a primer on constructing stochastic processes on Riemannian manifolds.

\subsection{Calculus on Riemannian manifolds}\label{ssec:riemann-manifolds}

Let ($\mathcal{M}$, g) be a $d$-dimensional Riemannian manifold equipped with the metric $g$. We assume $\mathcal{M}$ is complete, connected and has no boundary. For each $z\in \mathcal{M}$, $g$ is defined as an inner product on the tangent space $T_z\mathcal{M}$ \citep{do1992riemannian}:
\begin{equation}\label{eq:riemannian-inner-prod}
    \langle U,V\rangle_g := \sum_{i,j=1}^d g_{ij}U^iV^j,
\end{equation}
for tangent vectors $U ,V \in T_z\mathcal{M}$, where $g_{ij}$ denotes the metric tensor in local coordinates.
The gradient of a function $b\in\mathcal{C}^\infty(\mathcal{M})$ in local (manifold-centric) basis $(\partial_1,\dots,\partial_d)\in T_z\mathcal{M}$ is given by
\begin{equation}\label{eq:riemannian-gradient}
    \left(\nabla_\mathcal{M}b\right)^i := \sum_{j=1}^dg^{ij}\partial_jb,
\end{equation}
where $g^{ij}:=g_{ij}^{-1}$ is the Riemannian co-metric. 
The associated Laplace-Beltrami operator is  
\begin{equation}\label{eq:laplace-beltrami}
    \Delta_{\mathcal M} b := \operatorname{div}(\nabla_{\mathcal M} b).
\end{equation}

To work with data densities, we further assume $\mathcal{M}$ is oriented and admits a volume form $\unif$. This volume form induces a measure on $\mathcal M$, with total volume $\textnormal{Vol}(\mathcal{M}) := \int_\mathcal{M}\unif$.

We will call variables in local coordinates, e.g., $z:=(z^1, \dots , z^d)\in\mathcal{M}$ as intrinsic and the corresponding differential operators $\nabla_\mathcal{M},\Delta_\mathcal{M}$ will define the intrinsic geometry. Meanwhile, variables $x:=(x^1, \dots, x^n)\in\mathbb{R}^n$ are termed extrinsic and obtained through an isometric embedding\footnote{$F$'s existence is guaranteed by Nash’s isometric embedding theorem for large enough $n$.} into ambient Euclidean space, $F:\mathcal{M}\to \mathbb{R}^n$, equipped with the usual differential operators $\nabla, \Delta$. %See Appendix \ref{apdx:riem-geom} for further details. 

\subsection{Stochastic processes on manifolds from infinitesimal generators}\label{ssec:generators}

We now discuss the construction of stochastic processes on $\mathcal{M}$ by using differential operators. 
The correspondence between $\Delta_\mathcal{M}$ and diffusions on $\mathcal{M}$ is well-known \citep{stroock2007multidimensional}. Specifically, for any test function $f\in \mathcal C^\infty(\mathcal{M})$, the SDE from Eq. \eqref{eq:manifold-diffusion} has an associated partial differential equation 
\begin{equation}
    \partial_tf = Lf,
\end{equation}
where $L$ is called the \emph{infinitesimal generator} 
\begin{equation}\label{eq:manifold-generator}
    L = \frac{1}{2}\sum_{i,j=1}^d a^{ij}(z)\frac{\partial}{\partial z^i}\frac{\partial}{\partial z^j} + \sum_{i=1}^db^i(z)\frac{\partial}{\partial z^i},
\end{equation}
with $b(z)$ a drift field and $a(z)=\sigma(z) \sigma(z)^\top$ the noise covariance restricted to $T_z\mathcal{M}$.
The Laplace-Beltrami operator $\frac{1}{2}\Delta_\mathcal{M}$ is a special case for $b(z)=0$, $a(z)=\mathbb{I}$, and corresponds to an isotropic Brownian motion $dZ_t = dB_t^\mathcal{M}$ on $\mathcal{M}$ whose invariant law is the volume measure, $\unif$ \citep{hsu2002stochastic, grigoryan2009heat, bakry2013analysis}.

However, this classical theory does not generalise to arbitrary manifolds, where $\Delta_\mathcal{M}$ is typically unknown. 
An alternative is offered by diffusion geometry \citep{jones2024diffusiongeometry}, which estimates differential operators directly from data. 
The central object in diffusion geometry is the 
%We now consider the extrinsic embedded manifold $F(\mathcal{M})$. 
\emph{carré-du-champ} (CDC)
\begin{equation}\label{eq:CDC-def}
\Gamma(f,g)=L(fg)-f {L}g-g {L}f, \quad \Gamma(f,f) \equiv \Gamma(f)
\end{equation}
which relates the metric structure of the data to the infinitesimal generator $L$ of a stochastic process acting on functions $f, g\in C^\infty(\mathcal{M})$. Specifically, for $L=\Delta_\mathcal{M}$, $\Gamma$ reduces to the Riemannian metric $\Gamma_{\Delta_\mathcal{M}}(f,g) = g(\nabla_\mathcal{M} f, \nabla_\mathcal{M} g)$ \citep{jones2026computing}.
%The following proposition formalizes the geometric intuition %the isometric embedding $F$ induces an inner product in ambient space, $\langle U,V\rangle_g = \langle \dd F(U), \dd F(V)\rangle_{\mathbb{R}^n}$ and the CDC encodes this metric via $\Gamma(f,h)=g(\nabla_\mathcal{M} f,\nabla_\mathcal{M} h)$. In extrinsic coordinates $(x^1, \dots, x^n)$, 
Further, the CDC acts as a tangent space projector in ambient coordinates.
%We formalize this geometric intuition in Proposition \ref{prop:CDC-projector}.
\begin{proposition}\label{prop:CDC-projector}
    Let $\Gamma$ be the CDC operator associated with $\Delta_\mathcal{M}$ as per Eq. \eqref{eq:CDC-def}. Then, for ambient coordinates $(x^1, \dots x^n)$, $\Gamma^{ij}(x)=\Gamma(x^i,x^j)=\langle \pi_x(e_i),\pi_x(e_j)\rangle$ at a point $x\in F(\mathcal{M})$ is a projection matrix onto $T_x\mathcal{M}$. That is, we have 
    \begin{equation}
        \Gamma^{ij}(x) = \langle \pi_x(e_i), \pi_x(e_j) \rangle,
    \end{equation}
    where $\pi_x$ is the orthogonal projector $\pi_x:\mathbb{R}^n\to T_xF(\mathcal{M})$
\end{proposition}

%The covariance structure $\sigma$ from Eq. \eqref{eq:manifold-diffusion} is captured by the 
As shown in Ref. \cite{hsu2002stochastic}, this property allows using $\Gamma$ to lift the intrinsic operator $L$ to an equivalent ambient-space operator $\tilde{{L}}$. % whose trajectories remain on $F(\mathcal{M})$ when initialized on $F(\mathcal{M})$. 
\begin{theorem}[\citep{hsu2002stochastic}]\label{thrm:L-extension}
Let $x:=(x^1, \dots , x^n)$ be ambient coordinates, $f^i(x):=x^i$ be the coordinate function on  $\mathcal{M}$, 
%\begin{equation}\label{eq:ambient-SDE-coeffs}
%    \tilde{a}^{ij} := \Gamma(f^i,f^j), \quad \tilde{v}^i := Lf^i,
%\end{equation}
%where $\Gamma$ is as in Eq. \eqref{eq:CDC-def}, 
and define the operator
% \begin{equation}\label{eq:ambient-diffusion-generator}
%     \tilde{L} := \frac{1}{2}\sum_{i,j = 1}^n\tilde{a}^{ij}\frac{\partial^2}{\partial x^i \partial x^j} + \sum_{i=1}^n\tilde{v}^i\frac{\partial}{\partial x^i}.
% \end{equation}
\begin{equation}\label{eq:ambient-diffusion-generator}
    \tilde{L} := \frac{1}{2}\sum_{i,j = 1}^n\Gamma(f^i,f^j)\frac{\partial^2}{\partial x^i \partial x^j} + \sum_{i=1}^nLf^i\frac{\partial}{\partial x^i}.
\end{equation}

Then, for any $\tilde{f}\in\mathcal{C}^\infty(\mathbb{R}^n)$ that extends $f\in\mathcal{C}^\infty(\mathcal{M})$ to ambient coordinates, namely $f=\tilde{f}\circ F$ for an isometric embedding $F$, we have $\tilde{L}\tilde{f} = Lf$.
\end{theorem}

\begin{remark}
    The operator $\tilde L$ should be understood as an ambient representation of the intrinsic generator $L$; not as a new generator on all of $\mathbb{R}^n$. In other words, $\tilde L \tilde f$ is only defined through its restriction to $F(\mathcal{M})$, where
    \(
    (\tilde L\tilde f)\big|_{F(\mathcal{M})} = L\big(\tilde f\big|_{F(\mathcal M)}\big).
    \)
    Thus, $\Gamma$ determines the tangential noise, while the first-order term $L$ reproduces the intrinsic manifold dynamics.
\end{remark}

According to Theorem \ref{thrm:L-extension}, the stochastic process $Z_t$ on $\mathcal{M}$ (Eq. \eqref{eq:manifold-diffusion}) is equivalent to an $\mathbb R^n$-valued diffusion process $Y_t$ in $F(\mathcal{M})$. Namely,
% \begin{equation}\label{eq:ambient-manifdold-SDE}
%     d {Y}_t = \left(\tilde{v}({Y}_t) + \tilde{b}(Y_t)\right)dt + \tilde{a}^{1/2}(Y_t)d W_t,
% \end{equation}
\begin{equation}\label{eq:ambient-manifdold-SDE}
    d {Y}_t = \left(L + \tilde{b}(Y_t)\right)dt + \Gamma^{1/2}(Y_t)d W_t,
\end{equation}
where $W_t$ is a standard Brownian motion in $\mathbb R^n$ and $\tilde{b}:F(\mathcal M)\to \mathbb R$ is a smooth extension of $b:\mathcal M \to \mathbb R$. Importantly, this ambient space representation of the infinitesimal generator will allow us to construct it directly from data.

\section{Data-driven construction of diffusions on implicit manifolds}\label{sec:data-driven-SDE}

We now turn to estimating $L$ from a point cloud of i.i.d. samples $X_N := \{x_i\}_{i=1}^N \subset \mathbb R^n$ drawn from the volume measure $\unif$. Our construction is based on random walks on graphs. For a global neighbourhood radius $\epsilon >0$, we take the symmetric unweighted adjacency matrix $W_{ij}$, which can either be taken as a $W_{ij}:= \mathbf{1}\left\{ \|x_i-x_j\|_2^2 \leq \epsilon \right\}$ or a more general class of symmetric normalized kernels (see Appendix \ref{apdx:graph-laplacian-convergence}).
%which we symmetrize by imposing $W_{ij} \overset{!}{=} W_{ji}$. 
%and define the associated row-stochastic random walk matrix as
%\begin{equation}\label{eq:RW-matrix}
%    P_{ij} := \frac{W_{ij}}{m_i},
%\end{equation}
We then define the random walk graph Laplacian (RWGL) \citep{trillos2018errorestimatesspectralconvergence}, 
\begin{equation}\label{eq:RWGL-action}
    (\rwgl_N u)(x_i) := \frac{c}{\epsilon^2}\sum_{i\sim j}\frac{W_{ij}}{m_i}\left(u(x_i)-u(x_j)\right), 
\end{equation}
scaled by $c/\epsilon^2$, $c>0$, where $u:X_N\to \mathbb R$ is a function over the nodes and $m_i := \sum_{j=1}^N W_{ij}$ is the node degree. \citep{trillos2018errorestimatesspectralconvergence}.

Proving the convergence of $\rwgl_N$ to the corresponding continuous operator requires interpolating the graph Laplacian:
\begin{equation}\label{eq:smoothed-rwgl}
\mathsf{L} := \discrop_N^* \, \rwgl_N \, \discrop_N : H \to H,
\end{equation}
where $\mathsf{P}_N:H\to H_N$ and $\mathsf P_N^*:H_N\to H$ are identification maps between the Hilbert spaces of $L^2$-functions on the data $H_N := L^2(X_N,\mathbb{R})$ and on the manifold $H := L^2(\mathcal{M},\mathbb{R})$. Intuitively, $\mathsf P_N$ is already present through the action of the embedding $F$, and $\mathsf P_N^*$ acts as a nearest-neighbour approximator, mapping a function assignment $f(y)$ to $f(x_j)$, where $x_j = \argmin_{x\in X_N}\|x-y\|_2^2$ is the closest data point to $y\in\mathbb R^n$. 

%%%%%%%%%%%%%%%%%%%%%%%%%%%%%%%%%%%%%%%%%%%%%%%%%%%%%%%%%
%               Data Driven SDE
%%%%%%%%%%%%%%%%%%%%%%%%%%%%%%%%%%%%%%%%%%%%%%%%%%%%%%%%%
We may now define a data-driven version of the SDE \eqref{eq:ambient-manifdold-SDE}
\begin{equation}\label{eq:data-driven-SDE}
    d\mathsf{X}_t = \biggl(\mathsf{L}(\mathsf{X}_t) + b(\mathsf{X}_t)\biggr)d t + \mathsf{\Gamma}^{1/2}(\mathsf{X}_t)d W_t,
\end{equation}
where $\mathsf{\Gamma}$ is the CDC operator \eqref{eq:CDC-def} associated to $\mathsf{L}$ from Eq. \eqref{eq:smoothed-rwgl}. 
By Theorem \ref{thrm:L-extension}, Eq. \eqref{eq:data-driven-SDE} is the data-driven, ambient-valued SDE corresponding to Eq. \eqref{eq:manifold-diffusion} that smoothly extends the drift term $b$. 
The following result justifies the operator limit of the discrete generator induced by the graph Laplacian, $L_N$, to the Laplace-Beltrami operator $\Delta_\mathcal{M}$, and asserts that the trajectories induced by these generators converge in law on path space $C[0,T]$ to the corresponding manifold diffusion.

\begin{theorem}\label{thrm:mosco-conv}
    Given data points $x_i\overset{i.i.d.}{\sim}\unif$, let $\mathsf{L}$ be the interpolated random walk Laplacian of the graph over these points, as per Eq. \eqref{eq:smoothed-rwgl}. 
    In the limit $N\rightarrow\infty$ and $\epsilon(N) \rightarrow 0$ 
    %as per Eq. \eqref{eq:eps-decay}
    , $\mathsf{L}$ converges (in the strong resolvent sense) to $c\Delta_\mathcal{M}$, where $c$ is as in Eq. \eqref{eq:RWGL-action}. Further, the probability paths of $\mathsf{X}_t$ (Eq. \eqref{eq:data-driven-SDE}) converge in law to the associated manifold diffusion $Z_t$,
    \begin{equation}
        \mathsf{X}_t \implies Y_t, \quad \text{in $C([0,T],F(\mathcal{M}))$},
    \end{equation}
     %as $N\to\infty$ and $\epsilon(N)\to 0$
     where $Y_t = F(Z_{ct})$ is the extrinsic lifting of $Z_t$
     %the manifold diffusion 
     generated by
    \(
    L = b\cdot\nabla_\mathcal{M}+ \tfrac12 \Delta_\mathcal{M}
    \).
\end{theorem}
\begin{proof}
    See Appendix \ref{apdx:mosco-proof}.
\end{proof}
% \begin{remark}
% \end{remark}
%Theorem \ref{thrm:sampling-guarantees} guarantees that an Euler-Maruyama discretization of $\mathsf{X}_t$ remains geometrically faithful. 

\section{Numerical solution of IMDs}
\label{sec:discretization}
We propose the following Euler-Maruyama (E-M) scheme for the numerical solution of Eq. \eqref{eq:ambient-manifdold-SDE}
\begin{equation}\label{eq:data-driven-EM}
    \hbX_{(\ell+1)h} := \hbX_{\ell h} + \left(\rwgl(\hbX_{\ell h})+b(\hbX_{\ell h})\right)h + \mathsf{\Gamma}^{1/2}(\hbX_{\ell h})\sqrt{h}\xi_{\ell+1},
\end{equation}
where $\xi_{\ell+1}\sim\mathcal{N}(0,\mathbf{I}_n)$ is an independently drawn Gaussian, and $h>0$ is the step size such that $t\in[\ell h , (\ell+1)h)$ for some $\ell\in\mathbb N$. The following result demonstrates that Eq. \eqref{eq:data-driven-EM} is a consistent discretization in Wasserstein-2 distance that preserves the intrinsic geometric structure in the limit of large-sample size and infinitesimal step size $h$.
\begin{theorem}\label{thrm:sampling-guarantees}
    Let $Y_t$ be the diffusion process satisfying Eq. \eqref{eq:ambient-manifdold-SDE}, and $\hbX_{\ell h}$ given by Eq. \eqref{eq:data-driven-EM} with step size $h>0$. Assume that the coefficients $b,\mathsf{L}$ and $\mathsf \Gamma$ satisfy regularity conditions of Assumption \ref{ass:moment}. 
    
    Then, for any fixed $T>0$, and any $\eta >0$ and for all $t\in[\ell h , (\ell+1)h)$ with $h\ell \leq T$, we have 
    \begin{equation}
        W_2\left(\mathcal{L}(Y_t), \mathcal{L}(\hbX_{\ell h})\right) \leq \eta,
    \end{equation}
    as $N\to \infty, h\to 0, \epsilon \to 0$ under the assumed rates of Eq. \eqref{eq:eps-decay}  with probability $1-o_N(1)$.
\end{theorem}
\begin{proof}
    See Appendix \ref{apdx:sampling-proof}.
\end{proof}

%Together, Theorems \ref{thrm:mosco-conv} and \ref{thrm:sampling-guarantees} justify the data-driven construction on IMDs and their numerical solution in the limit $h\to 0$. 

%Numerical approximations of IMDs necessarily operate with finite step sizes, whereas the guarantees of Theorems \ref{thrm:mosco-conv} and \ref{thrm:sampling-guarantees} hold in the infinitesimal limit $h \to 0$.
\begin{wrapfigure}{r}{0.3\textwidth}
    \centering
    \vspace{-15pt}
    \begin{tikzpicture}[scale=0.7, >=latex]

        % -------------------------------------------------
        % styles
        % -------------------------------------------------
        \tikzset{
            manifold/.style={line width=1.5pt, gray!70},
            datapoint/.style={circle, fill=gray!45, inner sep=1.2pt},
            scorearrow/.style={->, line width=1pt, red!75},
            imdarrow/.style={->, line width=1pt, teal!75},
            botharrow/.style={->, line width=1pt, violet!80},
            lab/.style={font=},
            title/.style={font=\bfseries},
            box/.style={rounded corners=4pt, draw=black!15, fill=black!2}
        }

        % -------------------------------------------------
        % PANEL
        % -------------------------------------------------
        \begin{scope}[scale=1.0]
            %\node[title] at (0,3.1) {IMD + DRGD};

            % \draw[box] (-2.8,-2.2) rectangle (2.8,2.7);

            \draw[manifold, domain=-2.2:2.2, samples=200, smooth]
                plot (\x,{0.55*sin(1.4*\x r) + 0.18*\x});

            \foreach \x in {-1.8,-1.2,-0.5,0.2,0.9,1.6} {
                \fill[gray!50] (\x,{0.55*sin(1.4*\x r) + 0.18*\x}) circle (1.2pt);
            }

            % starting point on manifold
            \coordinate (x0) at (-1.0,{0.55*sin(1.4*(-1.0) r) + 0.18*(-1.0)});
            \fill[black] (x0) circle (1.5pt);
            \node[lab, below] at (x0) {\small $\hbX_\ell$};

            % off-manifold point after IMD step
            \coordinate (u1) at (0.7,0);
            \fill[black] (u1) circle (1.5pt);
            \node[lab, below right] at (u1) {\small $\mathsf{U}_{\ell+1}$};

            % retracted point on manifold
            \coordinate (x1) at (0.45,{0.55*sin(1.4*(0.45) r) + 0.18*(0.45)});
            \fill[black] (x1) circle (1.5pt);
            \node[lab, above left] at (x1) {\small $\hbX_{\ell+1}$};

            % IMD step
            \draw[imdarrow] (x0) -- (u1)
                node[midway, below=5pt, lab] {\small IMD};

            % score retraction
            \draw[scorearrow] (u1) -- (x1)
                node[midway, right=2pt, lab] { $\sigma^2s_\sigma$};
        \end{scope}
    \end{tikzpicture}
    \caption{We use the score of a pre-trained score model $s_\sigma(x) \approx \nabla_{x} \log p_\sigma(x)$ as a retraction toward the data manifold.}
    \label{fig:imd_score_decomposition}
    \vspace{12pt}
\end{wrapfigure}
\subsection{Numerical heuristics} \label{ssec:IMD-DRGD-discussion}
In our experiments, we found that the E-M discretization in Eq. \eqref{eq:data-driven-EM} typically led to bounded discretization errors and is a practically viable approach for simulating IMDs.
Nevertheless, we note that some applications may require large step sizes to improve numerical efficiency, which may lead to off-manifold deviations due to the accumulation of discretization error. 
As a numerical heuristic to control the error at large step sizes, E-M steps can be interleaved with a Denoising Riemannian Gradient Descent (DRGD) step \citep{kharitenko2025landingscoreriemannianoptimization} to retract off-manifold deviations. The DRGD step uses the learned score $s_\sigma$ at noise level $\sigma>0$ of a score-based diffusion model trained on the point cloud $X_N$ (Fig. \ref{fig:imd_score_decomposition}).
%may lead to off-manifold deviations due to discretization error, which can require very small step sizes to control. To alleviate this effect and allow tractable step sizes, we optionally apply a Denoising Riemannian Gradient Descent (DRGD) step \citep{kharitenko2025landingscoreriemannianoptimization} after each E-M update (see Fig. \ref{fig:imd_score_decomposition}).  Importantly, this correction is not part of the theoretical construction of IMDs, which are entirely defined by $\mathsf L$ and $\mathsf \Gamma$. Rather, the DRGD update serves as a practical stabilization mechanism and enables the use of larger step sizes. 
See Algorithm \ref{alg:IMD-with-drgd} for the numerical implementation and Appendix \ref{apdx:implementation-details} for more details.
\begin{algorithm}[h!]
    \caption{Implicit manifold-valued diffusions (with DRGD)}\label{alg:IMD-with-drgd}
    \KwInput{Point cloud $X_N=\{x_i\}_{i=1}^N$, step size $h>0$, iterations $L>0$, (optional drift $b$)}
    $s_\sigma(\cdot), \sigma \leftarrow$ from trained score model on $X_N$\hfill \COMMENT{(optional)}\\
    $\mathsf{L}, \mathsf{\Gamma} \leftarrow$ from Section \ref{sec:data-driven-SDE} \\
    \
    $\hbX_0 \leftarrow X_N[0]$ \hfill \COMMENT{Initialize on $X_N$}\\
    \For{$\ell=0, ..., L-1$}{
        Draw $\xi_{\ell+1}\sim\mathcal{N}(0,\mathbf{I}_{n})$ \\
        $\mathsf{U}_{\ell+1} = \hbX_\ell + \left[\mathsf{L}(\hbX_\ell) + b(\hbX_\ell)\right]h + \mathsf{\Gamma}^{1/2}(\hbX_\ell)\sqrt{h}\xi_{\ell+1}$ \hfill \COMMENT{E-M step}\\
        $\hbX_{\ell+1} = \mathsf{U}_{\ell+1} + \sigma^2 s_\sigma(\mathsf{U}_{\ell+1})$ \hfill \COMMENT{(optional) DRGD stabilization}\\
    }
    \KwResult{Trajectory $\{\hbX_i\}_{i=0}^L$}
\end{algorithm}
%\newpage
\vspace{-15pt}
\section{Experiments}\label{sec:experiments}
We evaluate IMDs on both synthetic and real-world data to highlight three key properties:
\begin{enumerate}[label=(\alph*)]
    \item \emph{Geometric fidelity.} IMDs follow the manifold for small time steps with bounded errors. \label{item:geom-fidelity}
    \item \emph{Statistical fidelity.} The law of the generated process produces statistically meaningful endpoints, allowing for sampling of probability distributions supported on $
    \mathcal M$.\label{item:stat-fidelity}
    \item \emph{On-manifold exploration.} IMDs yield tangential dynamics on the manifold, 
    %whereas score-based generative models yield dynamics in normal directions. 
    permitting smooth guided exploration between source-target pairs over the implicit manifold.\label{item:generalization}
\end{enumerate}
We demonstrate points \ref{item:geom-fidelity} and \ref{item:stat-fidelity} through synthetic experiments on the $d$-sphere $\mathbb S^d$, and point \ref{item:generalization} on a real-world image dataset. Further experimental details on the datasets (\ref{apdx:ssec:datadets}), the score model used for the DRDG step (\ref{ssec:more-on-SGM}) and the specific parameters of the E-M steps (\ref{apdx:ssec:exp_details}) are in the Appendix \ref{apdx:implementation-details}.
%We refer the reader to Appendix \ref{apdx:implementation-details} for details on the experimental setups. 

\subsection{Sampling on the hypersphere}\label{ssec:Langevin}
%In this section we evaluate IMDs against the analytically well-known stochastic process of Brownian motion on the $d$-sphere $\mathbb{S}^d$. 
\paragraph{Spherical Brownian motion} To quantify the geometric fidelity \ref{item:geom-fidelity}  of IMDs, we simulated Eq. \eqref{eq:data-driven-EM} on the hypersphere and computed the distance to manifold 
\begin{equation}
    \delta_t \in \vert\hbX_t - \pi_{\mathbb{S}^d}(\hbX_t)\vert, 
\end{equation}
where $\pi_{\mathbb{S}^d}$ is a projection onto $\mathbb{S}^d \subset \mathbb R^{d+1}$. We simulate trajectories for $d\in\{2,7\}$ and vary the step size $h\in\{10^{-3}, 10^{-4}, 2\times 10^{-5}\}$. We find that IMDs remain faithful to the manifold for long simulation times, as shown in Fig. \ref{fig:ballballball}a and quantified in Table \ref{tab:geom-fide}. Adding interleaved DRGD steps further improves fidelity, but they are not needed to constrain the dynamics on the manifold for small step sizes (Table \ref{tab:geom-fide}). By comparison, we conducted ablation experiments, where either the Laplacian term in Eq. \eqref{eq:data-driven-EM} was removed (CDC only, Fig. \ref{fig:ballballball}b), or used isotropic ambient noise instead of CDC-projected noise combined with DRGD retraction (Fig. \ref{fig:ballballball}c) obtaining worse manifold fidelity (Table \ref{fig:methods-comparison}, Appendix \ref{apdx:comparison-apdx}).

\begin{table}[h!]
\centering
\small
\setlength{\tabcolsep}{5pt}
\renewcommand{\arraystretch}{1.15}
\begin{tabular}{c c c c l}
\toprule
\textbf{Dimension} $d$ & \textbf{Step size} $h$ & \textbf{IMD} ($\mathbb{E}_t[\delta_t]$ / $\max_t\delta_t$) & \textbf{IMD+DRGD} ($\mathbb{E}_t[{\delta}_t]$ / $\max_t\delta_t$)  \\
\midrule
2 & $10^{-3}$ & $0.005 \pm 0.002 \,/\, 0.014 \pm 0.004$ & $0.002 \pm 0.001 \,/\, 0.004 \pm 0.001$  \\
\midrule
7 & $10^{-3}$ & $0.130 \pm 0.048 \,/\, 0.310 \pm 0.091$ & $0.003 \pm 0.001 \,/\, 0.013 \pm 0.002$ \\
7 & $10^{-4}$ & $0.104 \pm 0.050 \,/\, 0.273 \pm 0.092$ & $0.003 \pm 0.001 \,/\, 0.007 \pm 0.001$ \\
7 & $2\times10^{-5}$ & $0.096 \pm 0.043 \,/\, 0.262 \pm 0.082$ & $0.004 \pm 0.000 \,/\, 0.006 \pm 0.001$ \\
\bottomrule
\end{tabular}
\vspace{10pt}
\caption{\textbf{Geometric fidelity of IMD Brownian Motion on spheres of varying dimensions.} The optional DRGD retraction step of Algorithm \ref{alg:IMD-with-drgd} helps control $\delta_t$ as step size increases, supporting our heuristic from Section \ref{ssec:IMD-DRGD-discussion}. Cells report $\mathbb{E}_t[{\delta}]_t \,/\, \max_t \delta_t$, each as mean $\pm$ one standard deviation over $30$ independent runs, with time horizon $T=2$. See Appendix for details \ref{apdx:ssc-orthogonal-group}.}
\label{tab:geom-fide}
\vspace{-10pt}
\end{table}

\begin{figure}[h!]
\centering
\begin{minipage}{0.3\textwidth}
    \centering
    \includegraphics[width=\linewidth]{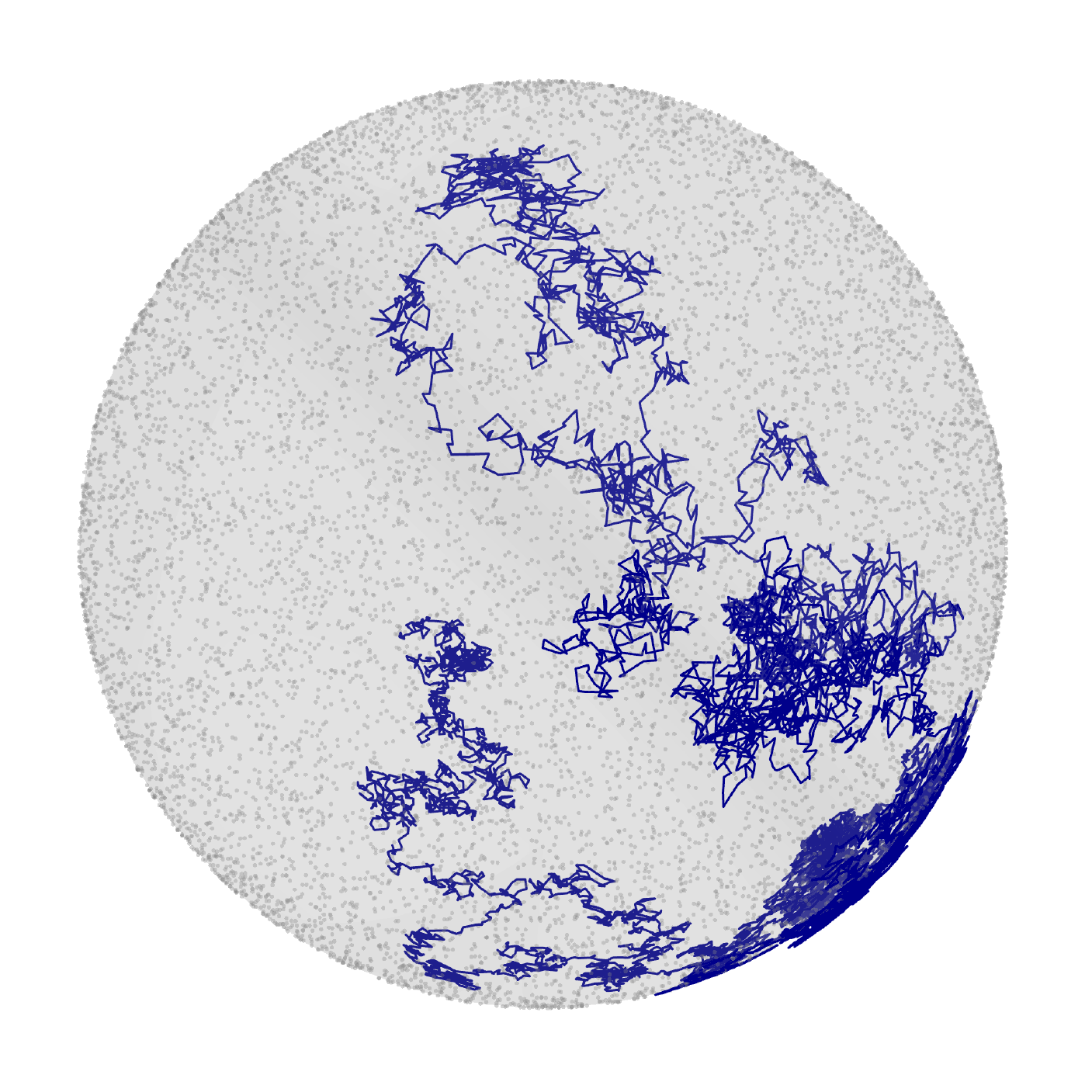}
    \subcaption{IMD}
\end{minipage}
\hfill
\begin{minipage}{0.3\textwidth}
    \centering
    \includegraphics[width=\linewidth]{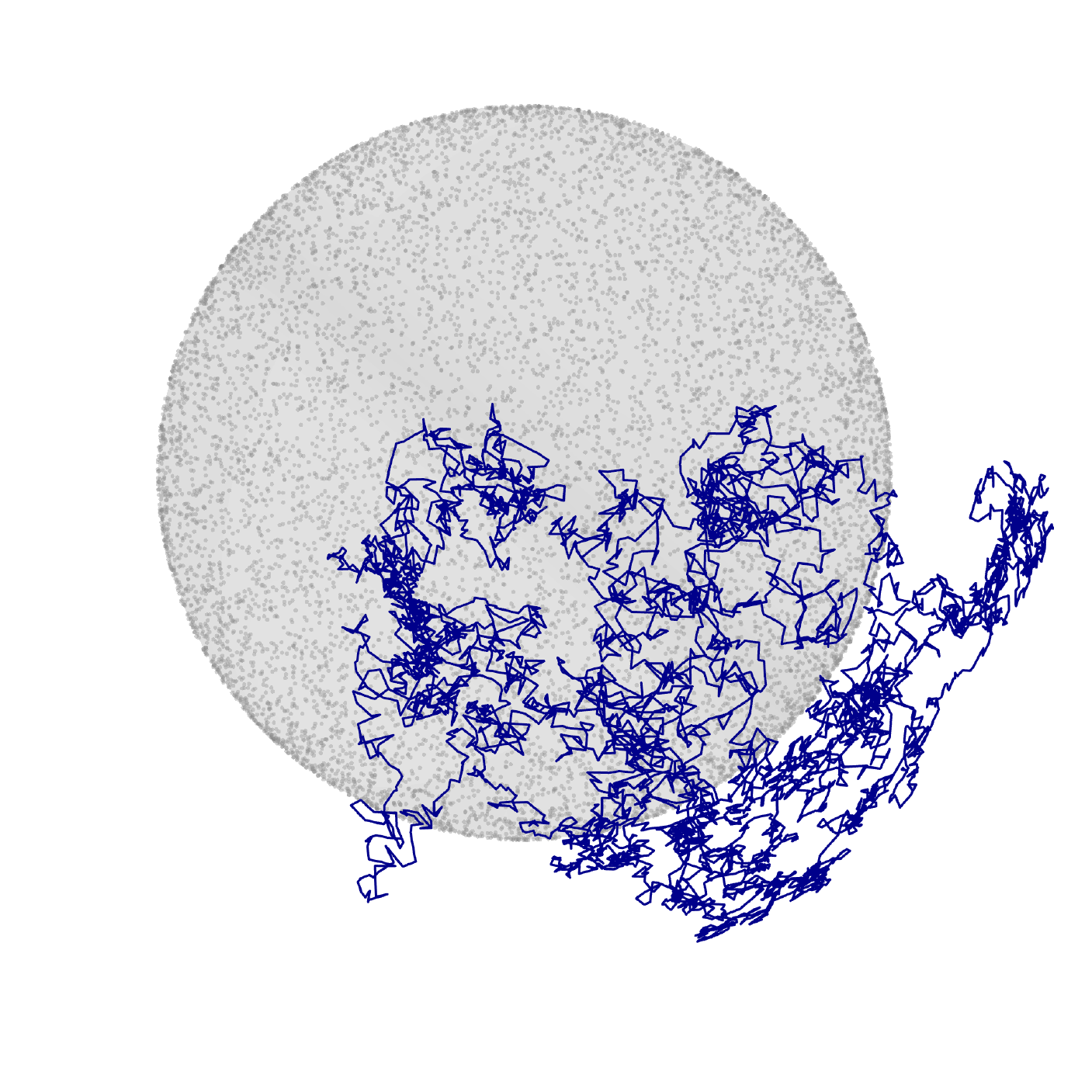}
    \subcaption{CDC only}
\end{minipage}
\hfill
\begin{minipage}{0.3\textwidth}
    \centering
    \includegraphics[width=\linewidth]{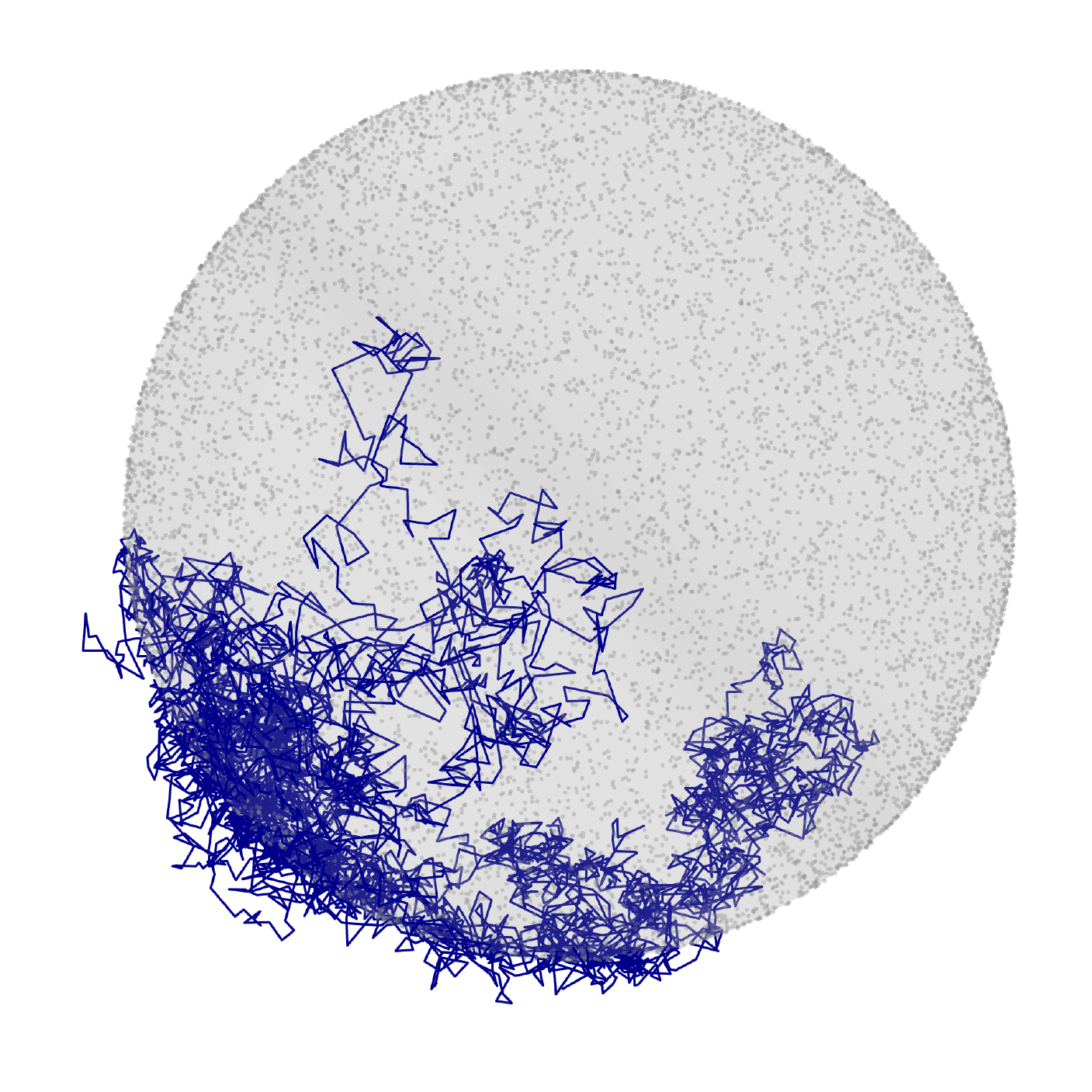}
    \subcaption{$\mathbb R^3$ Noise + DRGD Retraction}
\end{minipage}
%\begin{minipage}{0.23\textwidth}
    \caption{\textbf{Visual comparison of IMDs against  heuristic approaches.} (a) IMDs remain on $\mathbb S^2$. (b) CDC-projected noise without Laplacian drift correction accumulates errors. (c) Isotropic noise with DRGD retraction follows the manifold, but not accurately.}
    \label{fig:ballballball}
%\end{minipage}
\vspace{-10pt}
\end{figure}
%\vspace{-15pt}
\paragraph{Langevin dynamics}
On a manifold $\mathcal M$, Langevin dynamics with equilibrium density $q \propto e^{-U(x)}$, where $U:\mathcal{M}\to\mathbb R$ is a potential, satisfy the SDE \citep{chewi25logconcave, bakry2013analysis}
\begin{equation}\label{eq:manifold-langevin}
    %dY_t = -\nabla U(Y_t)dt + \sqrt{2}dB_t \quad \textnormal{resp.} \quad 
    dY_t = -\nabla_\mathcal{M} U(Y_t)dt + \sqrt{2}dB^\mathcal{M}_t.
\end{equation}
Following the definition of the Riemannian gradient as a tangent space projection \citep{boumal2023intromanifolds, kressner_low_nodate} and Proposition \ref{prop:CDC-projector}, we have $\nabla \tilde{U}(x) = \Gamma(U(x))$, yielding the following extrinsic SDE   
\begin{equation}\label{eq:IMD-langevin}
    dY_t = \left( \mathsf{L}(Y_t) - \mathsf \Gamma(\nabla U(Y_t)) \right)dt + \sqrt{2}\mathsf{\Gamma}^{1/2}(Y_t)dW_t.
\end{equation}
Using IMD Langevin dynamics \eqref{eq:IMD-langevin} we sample from the von Mises-Fisher (vMF) distribution on  $\mathbb{S}^d$
\begin{equation}\label{eq:vMF-distr}
    p(\mathbf{x},\boldsymbol{\mu}, \kappa) \propto \exp\left(\kappa\boldsymbol{\mu}^\top\mathbf{x}\right),
\end{equation}
where $\boldsymbol{\mu}\in\mathbb R^n$ is the mean vector and  $\kappa>0$ is the concentration parameter \citep{mardia2009directional}. As per Eq. \eqref{eq:IMD-langevin}, we thus have $\nabla \tilde{U}(x) = -\kappa\boldsymbol{\mu}$. For fixed $\boldsymbol{\mu}$, the stationary solution of the Langevin equation \eqref{eq:manifold-langevin} on the hypersphere $\mathbb{S}^{d}$ is analytically given by the sufficient statistic \(\theta := \boldsymbol{\mu}^\top Y_T\), which follows the law 
\begin{equation}\label{eq:vMF-stat-law}
    p(\theta) = \kappa^{\frac{d-1}{2}}\left(\sqrt{\pi}\Gamma\left(\frac{d}{2}\right)I_{\frac{d-1}{2}}(\kappa)\right)^{-1}e^{\kappa \theta}(1-\theta^2)^{\frac{d-2}{2}},
\end{equation}
where $\Gamma(\cdot)$ is the Gamma function, and $I_d$ is the modified Bessel function \citep{mardia2009directional}.
Thus, to measure statistical fidelity \ref{item:stat-fidelity}, we may simulate probability paths of IMDs with DRGD (see Section \ref{ssec:more-on-SGM} for details) and compute an empirical distribution of $\theta$ at end points $T$. 
The results are shown in Fig. \ref{fig:Langevin-test} for $\mathbb S^7 \subset \mathbb R^8$, projected onto the first-coordinate basis vector $\boldsymbol{\mu} := e_1$, indicating that the simulated process recovers the target equilibrium law.

\begin{figure}[h!]
    \centering

    \begin{minipage}[c]{0.5\textwidth}
        \centering
        \includegraphics[width=\linewidth]{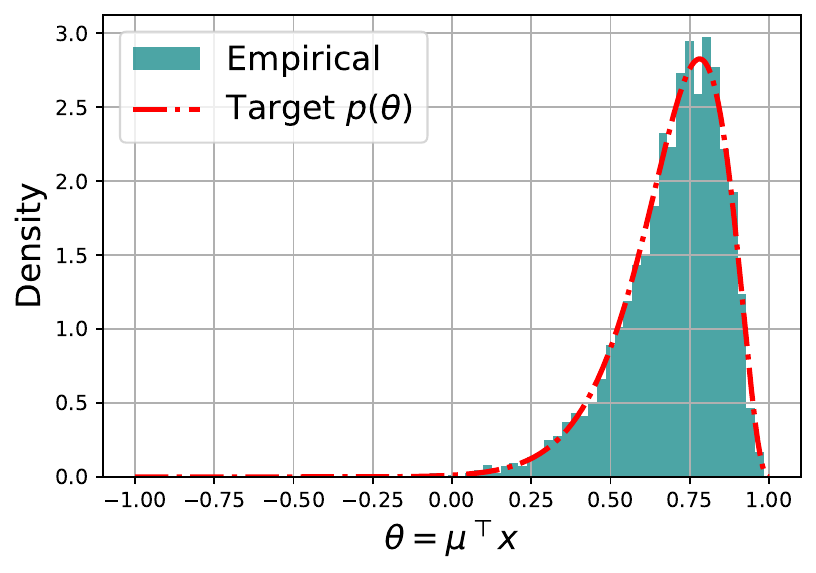}
    \end{minipage}
    \hfill
    \begin{minipage}[c]{0.40\textwidth}
        \caption{
            \textbf{Langevin sampling on $\mathbb S^7$.} Histogram of the endpoint statistic $\theta~=~\langle\boldsymbol{\mu},Y_T\rangle \in [-1,1]$ (teal) under Langevin dynamics computed with IMDs, compared with the theoretical density induced by the von Mises-Fisher distribution from Eq. \eqref{eq:vMF-stat-law} (red) on $\mathbb{S}^7\subset \mathbb R^8$. Parameters used: $\kappa=10$, $h=10^{-4}$.
        }
        \label{fig:Langevin-test}
    \end{minipage}
\vspace{-15pt}
\end{figure}

\subsection{Guided on-manifold exploration of image data}\label{ssec:MNIST-exp}
In real-world data, IMDs allow for stochastic exploration of implicit manifolds between a starting point and a target. 
To demonstrate this, we construct manifold-constrained Langevin dynamics 
over the MNIST dataset \citep{lecun2002gradient}. 
Specifically, we trained an autoencoder $E:\mathcal X\to\mathcal Z$ from pixel space $\mathcal X$ to latent space $\mathcal Z$, and will consider Langevin dynamics \eqref{eq:manifold-langevin} on the latent manifold $\mathcal Z$. 
For illustration, we then define a parabolic potential centered around $Z_\star\in\mathcal Z$, a latent corresponding to a data point:
\begin{equation}\label{eq:mnist-potential}
    U(Z_t) := \frac{\lambda}{2}\|Z_t-Z_\star\|_2^2, \quad \lambda>0 \textnormal{ a strength coefficient}.
\end{equation}
As a benchmark, we consider $\mathbb R^n$-valued Langevin dynamics, corrected with a DRGD retraction:
\begin{IEEEeqnarray}{rCl}\label{eq:DRGD-langevin}
    dY_t & = &  \textnormal{DRGD}_\sigma\bigl(-\nabla U(Y_t)dt + dW_t\bigr).
\end{IEEEeqnarray}
Note that by contrast to Eq. \eqref{eq:data-driven-SDE}, Eq. \eqref{eq:DRGD-langevin} lacks the Laplacian drift term and the CDC-based noise and gradient projection, replacing them by a DRGD retraction ($\textnormal{DRGD}_\sigma$,  Algorithm \ref{alg:IMD-with-drgd}, $\sigma>0$) and standard $\mathbb R^n$-valued  Brownian motion $W_t$. 
As a visual illustration, we transform the initial condition at digit `one' ($Z_0$) into the digit `seven' ($Z_\star)$, the minimum of the potential $U$. 
Fig. \ref{fig:MNIST-interpolation} confirms that IMD exploration yields a visually smooth transition between source and target image: first deforming the digit `one' within its own class, moving to a `seven' with no cross-bar,  and then adding a cross-bar to match the target (Fig. \ref{fig:MNIST-interpolation}a). 
By contrast, DRGD-aided Langevin exploration \eqref{eq:DRGD-langevin} yields an incoherent and discontinuous transition between source and target (Fig. \ref{fig:MNIST-interpolation}b,c).
Thus, tangential stochastic dynamics induced by IMDs provide a more locally coherent mode of exploration than solely score-driven dynamics, which do not guarantee on-manifold exploration.

%Visual comparisons supporting on-manifold dynamics \ref{item:generalization} are displayed in Fig. \ref{fig:MNIST-interpolation}, where we chose a perceptually meaningful example of . 
%We observe that the 
%Intuitively, one would expect a ``smooth'' trajectory to morph the starting ``1'' into a ``7 with no bar'', before then adding the bar. %Figure \ref{fig:MNIST-bent-manifold} schematically further illustrates this interpretation: IMDs evolve along the tangential geometry of the learned data manifold, whereas DRGD acts as a normal score-based correction toward the target region.

\definecolor{zoomteal}{RGB}{76,175,155}
\colorlet{framegray}{black!55}
\usetikzlibrary{calc,arrows.meta,fit}
\begin{figure}[hbtp]
\centering
\resizebox{\linewidth}{!}{%
\begin{tikzpicture}[
    font=\small,
    groupframe/.style={
        draw=framegray,
        line width=1.2pt,
        rounded corners=10pt,
        dash pattern=on 2.5pt off 2.5pt,
        inner xsep=0.28cm,
        inner ysep=0.22cm
    },
    zoomgroupframe/.style={
        draw=red,
        line width=1.4pt,
        rounded corners=10pt,
        dash pattern=on 2.5pt off 2.5pt,
        inner xsep=0.28cm,
        inner ysep=0.22cm
    },
    focusframe/.style={
        draw=red,
        line width=2pt,
        rounded corners=6pt,
        %dash pattern=on 2.5pt off 2.5pt
    },
    zoomarrow/.style={
        draw=red,
        line width=2.4pt,
        -{Latex[length=6mm,width=5mm]}
    },
    subcap/.style={
        font=\normalsize,
        align=center,
        text=black
    }
]

% ------------------------------------------------------------
% Parameters
% ------------------------------------------------------------
\def\FullW{18.5}
\def\ZoomW{18.5}
\def\RowGap{1.45}

% Vertical placement of the 3 double-row blocks
\def\TopY{0}
\def\MidY{-3.80}
\def\BotY{-7.55}

% ------------------------------------------------------------
% Top double row
% ------------------------------------------------------------
\node[anchor=north west, inner sep=0pt] (S0) at (0,\TopY)
  {\includegraphics[width=\FullW cm]{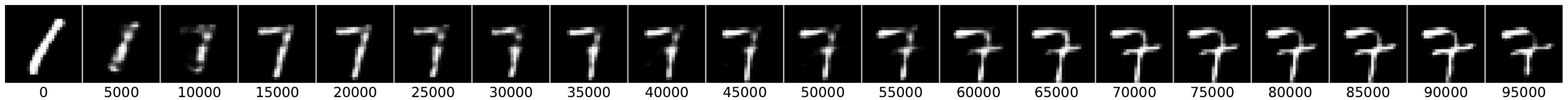}};

\node[anchor=north west, inner sep=0pt] (S1) at (0,{\TopY-\RowGap})
  {\includegraphics[width=\FullW cm]{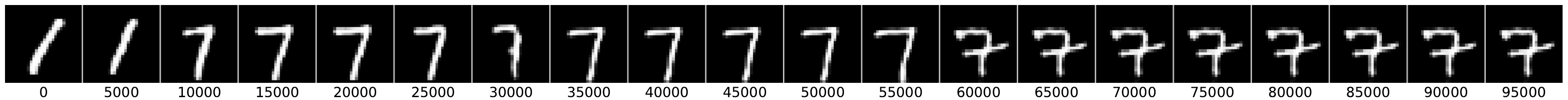}};

\node[groupframe, fit=(S0)(S1)] (G1) {};

% ------------------------------------------------------------
% Middle double row
% ------------------------------------------------------------
\node[anchor=north west, inner sep=0pt] (S2) at (0,\MidY)
  {\includegraphics[width=\FullW cm]{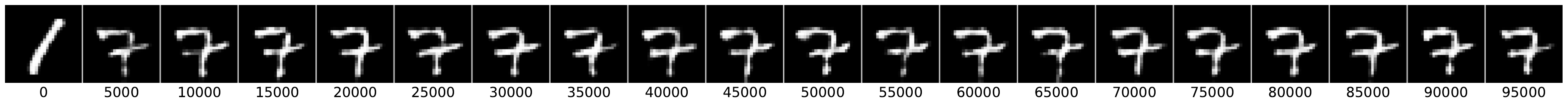}};

\node[anchor=north west, inner sep=0pt] (S3) at (0,{\MidY-\RowGap})
  {\includegraphics[width=\FullW cm]{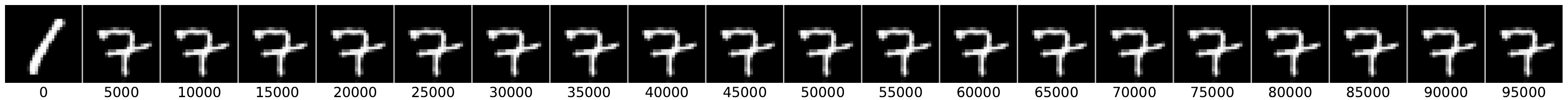}};

\node[groupframe, fit=(S2)(S3)] (G2) {};

% ------------------------------------------------------------
% Bottom zoomed double row
% ------------------------------------------------------------
\node[anchor=north west, inner sep=0pt] (Z0) at (0,\BotY)
  {\includegraphics[width=\ZoomW cm]{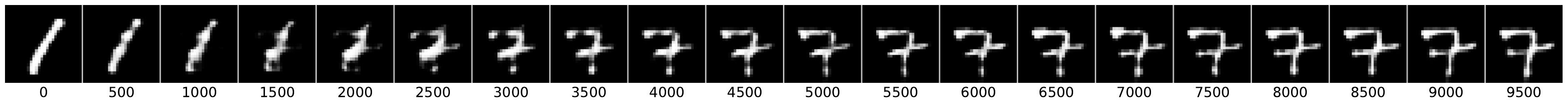}};

\node[anchor=north west, inner sep=0pt] (Z1) at (0,{\BotY-\RowGap})
  {\includegraphics[width=\ZoomW cm]{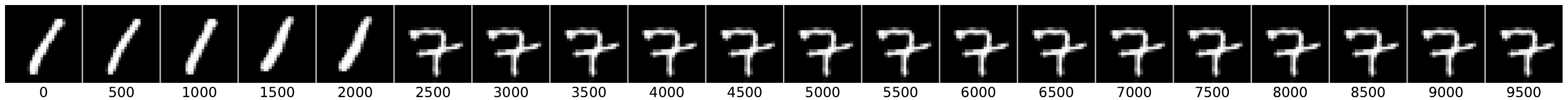}};

\node[zoomgroupframe, fit=(Z0)(Z1)] (G3) {};

% ------------------------------------------------------------
% Small teal zoom box on the middle top strip
% ------------------------------------------------------------
\coordinate (srcNW) at ($(S2.north west)+(0.03,-0.03)$);
\coordinate (srcSE) at ($(S2.north west)+({0.152*\FullW},-1.2)$);

\draw[focusframe] (srcNW) rectangle (srcSE);

% ------------------------------------------------------------
% Bent teal arrow
% ------------------------------------------------------------
\draw[zoomarrow]
  ($(srcNW)+(-0.10,-0.25)$)
    to[out=205,in=165,looseness=1]
  ($(G3.north west)+(-0.12,-0.10)$);

% ------------------------------------------------------------
% Subcaptions
% ------------------------------------------------------------
% Caption between first and second double-row blocks
\node[subcap] at ($(G1.south)!0.5!(G2.north)$)
  {(a) IMD trajectory (top) and associated nearest MNIST data point (bottom).};

% Caption between second and third double-row blocks
\node[subcap] at ($(G2.south)!0.5!(G3.north)$)
  {(b) DRGD trajectory (top) and associated nearest MNIST data point (bottom).};

% Optional caption below the bottom zoomed block
\node[subcap, anchor=north] at ($(G3.south)$)
  {(c) Zoomed-in DRGD trajectory from step 0 to $10'000$.};

\end{tikzpicture}%
}
\caption{\textbf{Potential-guided implicit manifold exploration.}
(a) IMD Langevin sampling \eqref{eq:IMD-langevin} starting from the digit `one' and driven by a potential with a minimum at the digit `seven'. The nearest data point at a particular time step is shown for visual guidance of the transition. The trajectory evolves continuously on the manifold, generating coherent transformation and consistent transitions.
(b) Standard Langevin dynamics with the same potential, retracted via DRGD.
(c) Zoomed in trajectory showing incoherent dynamics blending source and target early in the dynamics and rapid, discontinuous transition to the target.
}
%Comparison between  and a DRGD ambient Langevin sampling \eqref{eq:DRGD-langevin}, and their decoded nearest neighbours in $\mathcal{Z}$. The IMD trajectory (a) evolves more gradually along the latent data manifold, producing intermediate decoded images and nearest-neighbour data points that remain visually coherent throughout most of the transition. In contrast, the DRGD-corrected Euclidean trajectory (b) rapidly collapses toward the target region, producing transient source-target blends during the early part of the path. The zoomed-in DRGD trajectory (c) highlights this short-timescale collapse.  }
\label{fig:MNIST-interpolation}
\vspace{-10pt}
\end{figure}
\section{Discussion}\label{ssec:MNIST-discussion}
We introduced Implicit Manifold-valued Diffusions, a data-driven methodology for simulating diffusion processes on manifolds observed only through point clouds. We proved convergence of the induced process to its smooth manifold counterpart and demonstrated Langevin-type simulations on both synthetic and real-world data. The experiments indicate that tangential stochastic dynamics enables coherent on-manifold exploration beyond score-only retraction baselines, while the conjectured tangent-normal tradeoff points toward a broader theory of manifold-aware generative dynamics.

The experiments in Section \ref{ssec:Langevin} and \ref{ssec:MNIST-exp} corroborate the theoretical picture: IMDs recover target Langevin laws on spheres, remain geometrically stable at small step sizes, and produce coherent MNIST latent exploration.
%without requiring an additional score-based retraction in the exploration experiment. 
Yet, Table \ref{tab:geom-fide} also evidences the potential dimension dependence of the step size, $h(d)$, meaning that as $d$ increases, a smaller $h$ or DRGD-type correction is required to maintain geometric fidelity.
We formalize this intuition through the following conjecture.

%Revisiting the sphere experiments (Table \ref{tab:geom-fide}), we notice a tradeoff in the dimensionality $d$ and the step size $h$: higher dimensions seem to require smaller step sizes for accurate trajectories. We posit that dimension plays a role in accentuating curvature effects, which we present in the following conjecture.

%\vspace{-10pt}

\begin{conjecture}\label{conj:curvature}
We conjecture that the step size to dimension tradeoff is caused by dimension-dependent accumulation of normal errors. For a tangent Gaussian step \(\sqrt h \xi\in T_x\mathcal M\), the embedded exponential map satisfies
\begin{equation}\label{eq:exp-decomp}
    \exp_x(\sqrt{h}\xi) = x + \sqrt{h}\xi + \frac{h}{2}\Pi_x(\xi,\xi) + \mathcal{O}\left(\|\sqrt{h}\xi\|^3\right),
\end{equation}
where \(\Pi_x\) is the second fundamental form \citep{zhan2026riemannianlangevindynamicsstrong}.
Thus, finite-step normal error is controlled by $\Pi_x(\xi,\xi)$, which for isotropic Gaussians $\xi$ scales with contractions of the curvature over the $d$ tangent directions; this yields an $\mathcal{O}\left(dh\kappa_\mathcal{M}\right)$ normal scale \citep{do1992riemannian, zhan2026riemannianlangevindynamicsstrong}, where $\kappa_\mathcal{M}$ is a curvature-dependent term. 
Hence, finite-step geometric fidelity may depend on dimension, curvature and step size. See Appendix \ref{apdx:conjecture-support} for further details.
\end{conjecture} 

%\vspace{-10pt}
\paragraph{Addressing the curse of dimensionality} The computation of the graph Laplacian relies on a kernel $k_\epsilon$, which inherently suffers from the curse of dimensionality \citep{vapnik2013nature}. Moreover, the repeated matrix square-root computation at each time step makes IMDs prone to scalability issues for higher-dimensional problems. We however believe that an efficient computation of the Laplace-Beltrami and CDC operators is achievable through neural surrogates, as displayed in recent work \citep{bamberger2026riemannian}.

%\vspace{-10pt}
\paragraph{Combining IMDs and score-based models for generative modeling} Beyond efficient numerical schemes, the interplay of tangent (IMD) and normal directions (DRGD) could further our current understanding of the generative dynamics of score-based models in the low-noise regime. 
Indeed, studying the relationship between $\nabla\log p_\sigma$, $L$, and $\Gamma$ is currently an emerging line of research, blending geometry and stochastic dynamics \citep{li2025scores, pidstrigach2022score}. %We also believe that it holds great practical potential, in investigating local flatness of data manifolds, and optimizing the balance between IMD and score-based contributions.
%\paragraph{Generative modeling}  
%Beyond theoretical understanding of diffusion models, we see a space for new generative architectures. 
Our results suggest that combining both IMDs and DRGD 
%to consider noisy data 
opens the door to implicit manifold-constrained generative models. Furthermore, extending IMDs to endpoint-conditioned processes (e.g. via Doob-$h$ transforms \citep{corstanje2026simulating, pidstrigach2025conditioning} or Schrödinger bridges \citep{de2021diffusion}) may offer computationally inexpensive and stable alternatives for exploring data manifolds.

%\vspace{-10pt}
\paragraph{Impact statement} This work aims to advance the theoretical foundations of the rapidly growing field of generative modeling. The main societal risks are therefore indirect and inherited from generative modeling more broadly -- such as misuse, biased data-driven outputs, or overconfident usage in sensitive scientific pipelines -- rather than specific to the proposed mathematical construction.
\paragraph{Code availability}The code will be released upon acceptance of this manuscript. We nonetheless provide access to all used parameters and methods in both the main text and appendices.  

\begin{ack}
The authors would like to thank Antonio Terpin and Friso de Kruiff for their valuable feedback on the manuscript, as well as Drew Tyler for helping us with last-minute technical issues. 

P.V. and V.K.B. acknowledge funding from the Swiss National Science Foundation (SNSF), grant number SNF 200021-232277. A.G. and C.G. were supported by an ERC grant (NEURO-FUSE, Project DOI: 10.3030/101163046).
\end{ack}

\newpage
\bibliography{refs}
\bibliographystyle{plainnat}
\newpage
\appendix
\section{Aesthetically Pleasing Plots}

\begin{figure}[htbp]
\centering
\begin{minipage}[t]{0.30\textwidth}
    \centering
    \includegraphics[width=\linewidth]{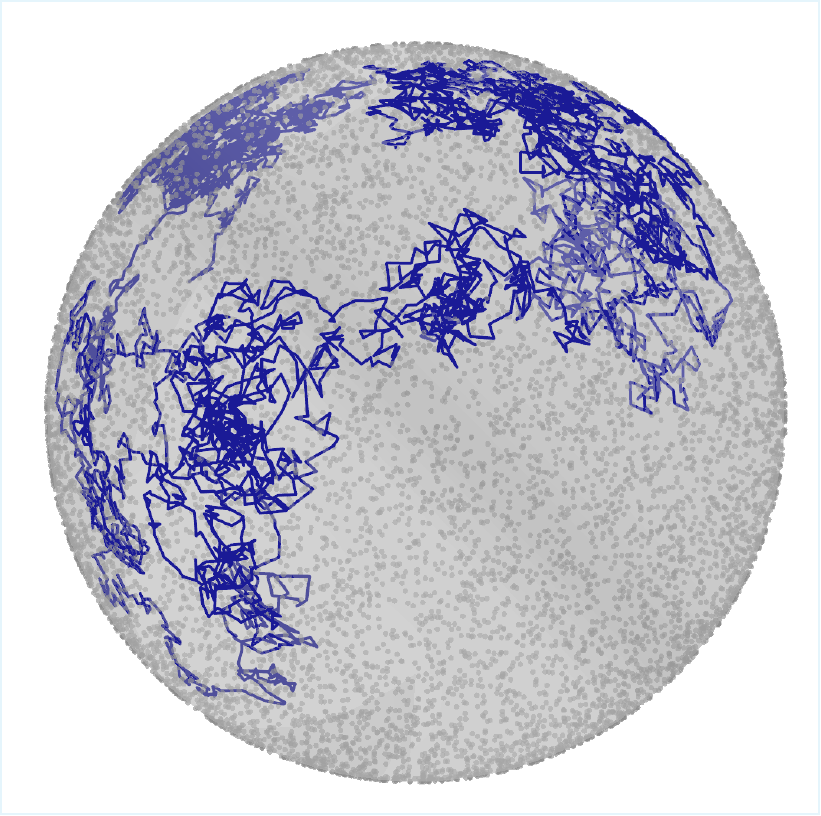}
    \subcaption{2-dimensional sphere}
    \label{fig:S2-BM}
\end{minipage}
\hfill
\begin{minipage}[t]{0.30\textwidth}
    \centering
    \includegraphics[width=\linewidth]{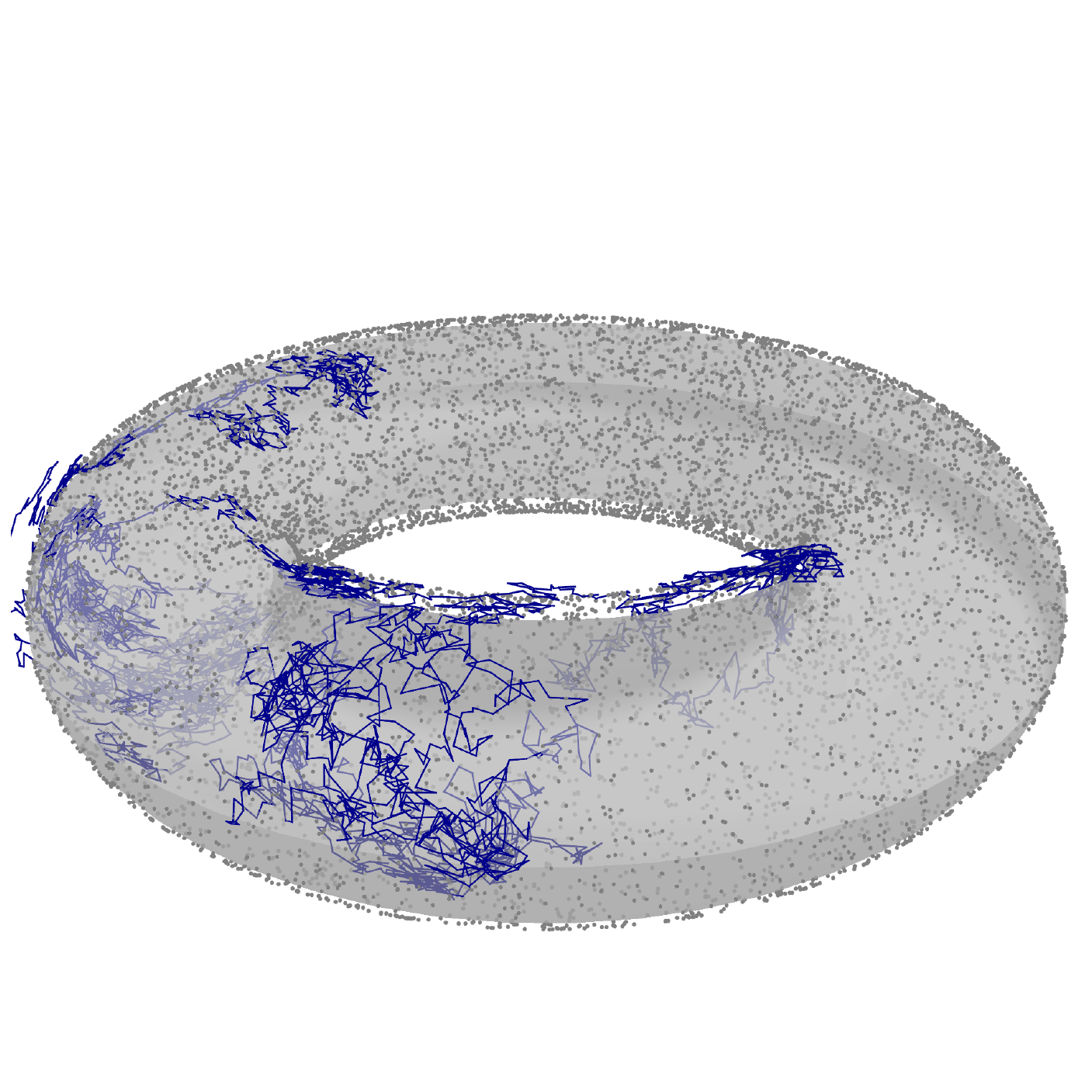}
    \subcaption{The Torus $\mathbb{T}^2$}
    \label{fig:torus-2}
\end{minipage}
\hfill
\begin{minipage}[t]{0.30\textwidth}
    \centering
    \includegraphics[width=\linewidth]{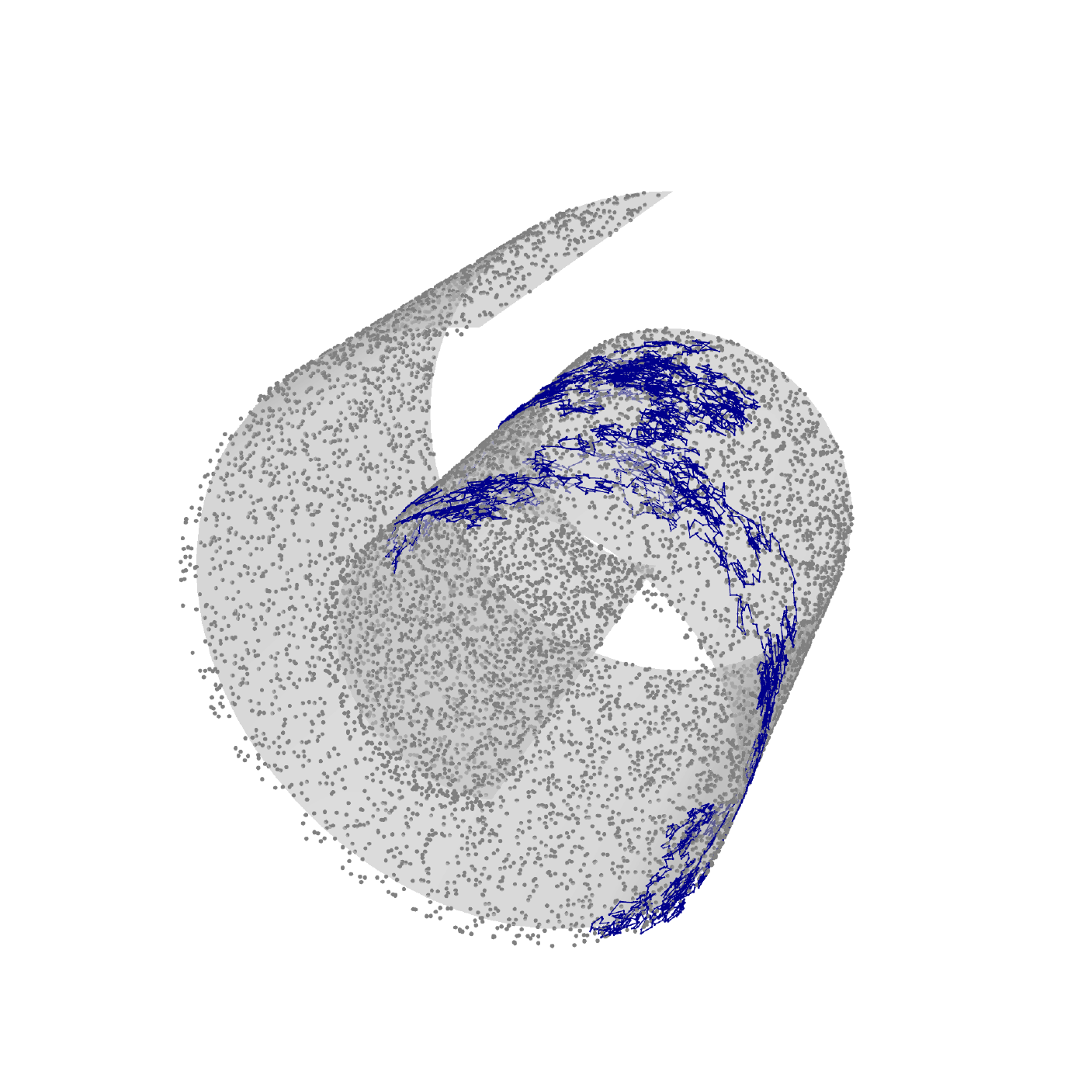}
    \subcaption{Swiss Roll}
    \label{fig:SR-BM}
\end{minipage}

\caption{IMD implementations in 3-dimensional space, with size is $10^{-3}$. We notice that while the Swiss Roll is not boundaryless, the nearest-neighbour approach to estimating $\mathsf{L}(X_t)$ via $\discrop^*_N$ prevents off-manifold drift.}
\label{fig:further-geometries}
\end{figure}

\section{Some Extrinsic Riemannian Geometry}\label{apdx:riem-geom}

We consider the case of a compact boundaryless Riemannian manifold $(\mathcal{M}, g)$ and an isometric embedding $F:\mathcal{M}\to \mathbb{R}^n$ and will study some properties of the embedded manifold $F(\mathcal{M})$. 
%\todo{explain a bit vector fields and differentials.}
We begin with some notation: let $(e_1, \dots e_n)$ denote the canonical basis of the ambient space $\mathbb{R}^n$, $(x^1, \dots x^n)$ denote an arbitrary coordinate space in $\mathbb{R}^n$, and $f\in\mathcal{C}^\infty(\mathcal{M})$ with an extension $\Tilde{f}\in \mathcal{C}^\infty(\mathbb{R}^n)$ such that $f = \Tilde{f}$ on $F(\mathcal{M})$. 
\begin{proposition}[\cite{kressner_low_nodate}]\label{prop:riemannian-gradient-proj}
    Let $\nabla \Tilde{f}$ denote the usual gradient of $\Tilde{f}$, and $\nabla_\mathcal{M}f$ the Riemannian gradient of $f$. Then for a point $x\in F(\mathcal{M})$, we have
    \begin{equation}
        \nabla \Tilde{f}(x) = P_{T_x\mathcal{M}}\nabla \Tilde{f}(x)+P^\perp_{T_x\mathcal{M}}\nabla \Tilde{f}(x).
    \end{equation}    
    In particular, this implies that the Riemannian gradient of $f$ at $x$ is the projection of $f$'s extension onto the tangent space at $x$: $\nabla_\mathcal{M}f(x) = P_{T_x\mathcal{M}}\nabla\Tilde{f}(x)$.
\end{proposition}

\subsection{Proof of Proposition \ref{prop:CDC-projector}}
    We begin with the intrinsic formulation of the Riemannian gradient for $f\in\mathcal{C}^\infty(\mathcal{M})$ at $p\in\mathcal{M}$:    \begin{equation}
        \langle \nabla_\mathcal{M}f(p), v\rangle_g = \dd f_p(v) 
    \end{equation}
    for $v\in T_p\mathcal{M}$, and where $\dd f_p(v)$ is the differential of $f$ at $p$ applied to $v$. When considering $f^i(x) = x^i$ the coordinate function, we have that 
    \begin{equation}
        \langle \nabla_\mathcal{M}x^i(p),v\rangle = v^i = \langle e_i,v\rangle.
    \end{equation}
    By decomposition of $e_i = P_{T_p\mathcal{M}}(e_i) +P^\perp_{T_p\mathcal{M}}(e_i)$ and Proposition \ref{prop:riemannian-gradient-proj}, we have
    \begin{equation}
        \nabla_\mathcal{M}x^i(p) = P_{T_p\mathcal{M}}(e_i)
    \end{equation}
    by the Riesz representation theorem on Hilbert spaces \citep{lax2014functional}. Thus, the CDC has entries $\left(\Gamma^{ij}(p)\right)^{k\ell}$ as a projection onto $T_p\mathcal{M}$.
\begin{flushright}
    \qed
\end{flushright}
\subsection{Proof of Theorem \ref{thrm:L-extension}}
The proof is quite straightforward.
We begin with the action of $L$ on $f$ by first writing the partial derivatives of $f(x) = \tilde{f}(f^1(x), \dots , f^n(x))$:
\begin{IEEEeqnarray}{rCl}
    \partial_i f & = & \sum_\alpha \partial_\alpha \tilde{f}\partial_if^\alpha, \\
    \partial_{ij}f & = & \sum_\alpha \partial_\alpha \tilde{f}\partial_{ij}f^\alpha + \sum_{\alpha,\beta}\partial_{\alpha\beta}\tilde{f}(\partial_if^\alpha)(\partial_jf^\beta).
\end{IEEEeqnarray}
Thus, in Einstein notation
\begin{IEEEeqnarray}{rCl}
    Lf & = & \frac{1}{2}a^{ij}\left(\sum_\alpha \partial_\alpha \tilde{f}\partial_{ij}f^\alpha + \sum_{\alpha,\beta}\partial_{\alpha\beta}\tilde{f}(\partial_i f^\alpha)(\partial_jf^\beta) \right) + b^i\left(\sum_\alpha\partial_\alpha\tilde{f}\partial_if^\alpha\right)\nonumber \\
    & = & \frac{1}{2}\sum_\alpha a^{ij}\partial_{ij}f^\alpha\partial_\alpha \tilde{f} + b^i\sum_\alpha\partial_\alpha\tilde{f}\partial_if^\alpha+\frac{1}{2}\sum_{\alpha,\beta}a^{ij}\partial_{\alpha\beta}\tilde{f}(\partial_if^\alpha)(\partial_jf^\beta) \nonumber \\
    & =& \sum_\alpha (Lf^\alpha)\partial_\alpha \tilde f^\alpha + \frac{1}{2}\sum_{\alpha,\beta}\Gamma(f^\alpha,f^\beta)\partial_{\alpha\beta}\tilde{f} = \tilde{L}\tilde{f}
\end{IEEEeqnarray}
on $\mathcal{M}$. $L$ acts on the manifold $\mathcal{M}$, where $\partial_i$ are along the tangent directions of $\mathcal{M}$. $\tilde{f}$ is defined wrt basis vectors of $\mathbb{R}^n$, not $\mathcal{M}$, so the derivative is not along $T\mathcal{M}$, whence we must apply the chain rule to convert to ambient space.
\begin{flushright}
    \qed
\end{flushright}
\section{Singular Limits of Geometric Graphs}\label{apdx:graph-laplacian-convergence}

We introduce the necessary machinery to construct a graph from samples of $\mathcal{M}$, as is done in \cite{trillos2018errorestimatesspectralconvergence, trillos2015rate, trillos2018variational, garcia2016continuum}. 
Denote $\unif$ the volume form of $\mathcal{M}$, from which the samples $\{x_i\}_{i=1}^N$ are i.i.d. drawn. The empirical measure of the $N$ data points $\mathcal{M}_N := (x_1, \dots , x_N)$ is denoted 
\begin{equation}
    \nu_N := \frac{1}{n}\sum_{i=1}^N\delta_{x_i},
\end{equation}
and asymptotically weakly converges to $\unif$ \citep{varadarajan1958convergence}. The rate of convergence is captured by the following $\infty$-Wasserstein distance of the transport map $T:\mathcal{M}\to X$ such that $T_\#\unif = \nu_N$:
\begin{equation}\label{eq:infty-wasserstein-dist}
    \epsilon := \dd_\infty(\unif, \nu_N) := \inf_{T:T_\#\unif = \nu_N} \esssup_{x\in\mathcal{M}}\dd_\mathcal{M}(x,T(x)),
\end{equation}
where $\dd_\mathcal{M}$ is the geodesic distance on $\mathcal{M}$ \citep{do1992riemannian}. $T$ further induces a partition of $\mathcal{M} = \bigcup_{i=1}^NU_i$ with $U_i := T^{-1}\left(\{x_i\}\right)$ the pre-image of the transport and $\unif(U_i) = \frac{1}{n}$. \citep{trillos2018errorestimatesspectralconvergence} guarantees that with high probability 
\begin{equation}\label{eq:eps-decay}
    \epsilon \leq C \left(\frac{\log N}{N}\right)^{1/d}
\end{equation}
for $d\geq 3$, and 
\begin{equation}
    \epsilon \leq C\frac{(\log N)^{3/4}}{N^{1/2}}
\end{equation}
for $d=2$ hold, which will allow us to construct appropriate asymptotic convergence results. 

We further define a kernel function 
\begin{equation}
    k_\epsilon(x_i,x_j) = \eta\left(|x_i-x_j|\right) \quad \textnormal{s.t.} \quad \int_{\mathbb{R}^m}\eta(|x|)\dd x = 1,
\end{equation}
which is decreasing, supported on $[0,1]$, and Lipschitz on $[0,1]$. The kernel defines defines the weights of the graph's edges 
\begin{equation}
    w_{ij} =  \frac{1}{Nh_N}\eta\left( \frac{|x_i-x_j|}{h_N} \right), 
\end{equation}
where $|\cdot|$ denotes Euclidean distance and $h_N$ is an appropriately selected neighourhood radius \cite{trillos2018errorestimatesspectralconvergence, tengler2026manifoldlimittrainingshallow}.
\begin{remark}
    While we consider a general kernel $\eta$, the hard-cutoff graph built using the adjacency matrix $W$ found in Eq. \eqref{eq:RWGL-action} in Section \ref{sec:data-driven-SDE} is recovered by the choice of $\eta(r) := \mathbf{1}_{r\leq 1}$, with bandwidth $h_N = \sqrt{\epsilon}$ up to normalization.
\end{remark}
The resulting graph is thus $\mathcal{G}_N := (\mathcal{M}_N, W_N)$. The \emph{surface tension} is defined as 
\begin{equation}
    \sigma_\eta := \int_{\mathbb{R}^m}|x\cdot e_1|^2\eta(|x|)\dd x,
\end{equation}
and is required for appropriate scaling of the graph Laplacian in order for convergence to be established.

We also introduce the \emph{contraction discretization} operator 
\begin{equation}
    \discrop_Nf:x_i \mapsto N\cdot\int_{U_i}f(x)d\mu(x)
\end{equation}

where $U_i$ is the pre-image of a point on $X_N$ in $\mathcal{M}$. Conversely we define 
\begin{equation}
    \discrop_N^*u:x\mapsto\Lambda\left(\sum_{i=1}^nu(x_i)\mathbf{1}_{U_i}(x)\right)
\end{equation}
the \emph{extension map}, where $\Lambda$ is a Lipschitz kernel ensuring smoothness. 

\begin{proposition}\label{prop:interp-discr-adj}
    For $\Lambda = \mathbb{I}$ the identity operator, $\discrop_N$ and $\discrop_N^*$ are adjoint in their respective spaces. 
\end{proposition}
\begin{proof}
    
We begin by letting $f\in H$ and $
    u\in H_n$, allowing us to compute the inner product 
    \begin{IEEEeqnarray}{rCl}
        \langle u,\discrop_N f\rangle_{H_N} := N\int_Xu\, \discrop f d(\mu_N) & = &  \sum_{i=1}^Nu(x_i)(\discrop f)(x_i) \nonumber \\
        & = & \sum_{i=1}^Nu(x_i)\int_{U_i}f(x)d\mu(x)
    \end{IEEEeqnarray}
    as well as 
    \begin{equation}
        \langle \discrop_N^*u, f \rangle_H = \int_\mathcal{M}\sum_{i=1}^Nu(x_i)\mathbf{1}_{U_i}f(x) d\mu(x) = \sum_{i=1}^Nu(x_i)\int_{U_i}f(x)d\mu(x)
    \end{equation}
    which matches the previous inner product over the graph under suitable conditions on $\rho$ and $\mathbf{m}$ (which are covered by \cite{trillos2018errorestimatesspectralconvergence}).
\end{proof}

\section{Proof of Theorem \ref{thrm:mosco-conv}}\label{apdx:mosco-proof}
\paragraph{Proof Outline}
The proof is based on the Kuwae--Shioya framework for convergence of Dirichlet forms on varying Hilbert spaces \citep{kuwaeshioya2003convergence}. In our setting, the discrete spaces \(H_N\) consist of functions on the nodes of the graph with \(N\) samples, while the limit space \(H\) is the corresponding continuum Hilbert space on \(\mathcal M\) (see Appendix~\ref{apdx:graph-laplacian-convergence} for the precise definitions). We then introduce the discrete and continuum Dirichlet forms \(\mathcal E_N\) and \(\mathcal E\), verify Mosco convergence, and conclude strong resolvent convergence of the associated infinitesimal generators.

\subsection{Setup}
We work with the discrete Hilbert spaces \(H_N\) and the continuum Hilbert space \(H\) introduced in Appendix~\ref{apdx:graph-laplacian-convergence}. The corresponding graph and manifold Dirichlet forms are
\begin{equation}\label{eq:graph-Dirichlet-form}
    \mathcal{E}_N(u,v) := u^\top \rwgl_N v, \quad \textnormal{for $u,v\in {H}_N$},
\end{equation}
and
\begin{equation}
    \mathcal{E}(f,h) := \int_\mathcal{M}f(x)\Delta_\mathcal{M}(h)(x)\unif(x).
\end{equation}
We write \(\mathcal E_N(u):=\mathcal E_N(u,u)\) and \(\mathcal E(f):=\mathcal E(f,f)\). We recall that identification operators \(\mathsf P:H\to H_N\) and \(\mathsf P^*:H_N\to H\) are used to compare discrete and continuum objects. We now recall a result pertaining to Dirichlet forms of graphs constructed as in Appendix \ref{apdx:graph-laplacian-convergence}.

\begin{lemma}[Lemmas 13 \& 14 \cite{trillos2018errorestimatesspectralconvergence}]\label{lemma:dirichlet-UB-LB}
    For the discrete and continuous Dirichlet forms defined above, the following hold
    \begin{enumerate}
        \item\label{lemma:dirichlet-form-UB} For all $f\in H^1(\mathcal{M})$,\footnote{Note that any $f\in H^1(\mathcal{M})$ is automatically in $\mathcal{C}(\mathcal{M})$.} we have
        \begin{equation}
            \mathcal{E}_N(\discrop_N f) \leq \underbrace{(1+C_1'\epsilon + C_2'\frac{\varepsilon}{\epsilon} + C_3'\epsilon^2)}_{=:\delta_N'}\mathcal{E}(f),
        \end{equation}
        and 
        \begin{equation*}
            C_1' = C\alpha L_p, \quad C_2' = C\left(d+\frac{2^{d+1}L_{k_\epsilon}(1+\alpha L_p)}{k_\epsilon(1/2)}\right), \quad C_3' = Cd(K+R^{-2})
        \end{equation*}
        and $C$ is a universal constant.
        \item \label{lemma:dirichlet-form-LB} For any $u\in {H}_N$, we have 
        \begin{equation}
            \mathcal{E}(\interp_Nu) \leq \underbrace{(1+C_1''\epsilon+ C_2''\frac{\varepsilon}{\epsilon} + C_3''\epsilon^2)}_{=:\delta_N''}\mathcal{E}_N(u),
        \end{equation}
        where $\interp_N$ is the interpolation map and 
        \begin{equation*}
            C_1'' = \alpha L_p, \quad C_2'' = C(d+C_2'), \quad C_3'' = (1+\frac{1}{\sigma_{k_\epsilon}})dK.
        \end{equation*}
    \end{enumerate}
\end{lemma}
\begin{corollary}\label{coro:delta-to-zero}
    We immediately notice that from Appendix \ref{apdx:graph-laplacian-convergence} $h_N\propto \sqrt{\epsilon}$, implying that $\epsilon,\frac{\varepsilon}{\epsilon}$ and $\epsilon^2$ all go to zero as  sample size $N\rightarrow\infty$, implying that all terms involving $C_i'$ and $C_i''$ go to zero as $N\rightarrow\infty$.
\end{corollary}

By Lemma \ref{lemma:dirichlet-UB-LB} and Corollary \ref{coro:delta-to-zero}, the comparison errors \(\delta_N',\delta_N''\) vanish as \(N\to\infty\). 

\subsection{Mosco Convergence} We begin by defining the notion of Mosco convergence of operators.

\begin{definition}\label{def:Mosco-conv}
    Let $X$ be a topological vector space, and $F_n:X\rightarrow[0,+\infty]$ be functionals on $X$. The sequence $(F_n)_n$ is said \textbf{Mosco convergent} to another functional $F:X\rightarrow[0,+\infty]$ if the following two conditions hold:
    \begin{itemize}
        \item for each sequence of $x_n\in X$ converging \emph{weakly} to $x\in X$ (denoted $x_n \rightharpoonup x$), we have 
        \begin{equation}
            \liminf_{n\rightarrow\infty}F_n(x_n) \geq F(n),
        \end{equation}
        \item for each $x\in X$ there exists an approximating sequence of elements $x_n\in X$ converging strongly to $x$, such that 
        \begin{equation}
            \limsup_{n\rightarrow\infty}F_n(x_n) \leq F(x).
        \end{equation}
    \end{itemize}
        We denote this $F_n \overset{M}{\rightarrow}F$. This is a stronger version of $\Gamma$-convergence, which doesn't require mixing the weak \& strong topologies. Mosco convergence is equivalent to both weak and strong $\Gamma$-convergence simultaneously. 
\end{definition}
We will now show that using Lemma \ref{lemma:dirichlet-UB-LB} and Corollary \ref{coro:delta-to-zero}, the Dirichlet forms of a geometric graph $\mathcal{G}=(X_N;E)$ whose nodes are iid sampled from $F(\mathcal{M})$ satisfy the conditions for Mosco convergence from Definition \ref{def:Mosco-conv}.

\paragraph{Lower Bound } We begin by stating that Dirichlet forms are by design convex, quadratic forms that are continuous in $H^1$, whence they are weakly {lower semicontinuous} (LSC). We now define a sequence of functions $u_n:V\rightarrow\mathbb{R}$ such that $\interp_n u_n \rightharpoonup f$ in $H$, then 
    \begin{IEEEeqnarray}{rCl}
        \mathcal{E}(f) &\overset{a)}{\leq}& \liminf_{N\rightarrow\infty} \mathcal{E}(\discrop^*_Nu_N) \nonumber \\
        & \overset{b)}{\leq} & \lim_{N\rightarrow\infty} (1+ \delta_N'')\mathcal{E}_n(u_N)
    \end{IEEEeqnarray}
    where we use weak LSC in $a)$ and Lemma \ref{lemma:dirichlet-UB-LB} point \ref{lemma:dirichlet-form-LB} in $b)$. We do note that the $\delta_N'$ and $\delta_N''$ terms go to zero as $N\rightarrow\infty$ by definition of the used rates in Eq. \eqref{eq:eps-decay}. We don't explain it here, though. 

\paragraph{Upper Bound} Conversely, we use Lemma \ref{lemma:dirichlet-UB-LB} point \ref{lemma:dirichlet-form-UB} to establish 
    \begin{IEEEeqnarray}{rCl}
        \limsup_{N\rightarrow\infty}\mathcal{E}_N(\discrop_Nf) & \leq & \limsup_{N\rightarrow\infty}(1+\delta_N') \mathcal{E}(f) \nonumber \\
        & = & \mathcal{E}(f)\lim_{N\rightarrow\infty}(1+\delta_N') = \mathcal{E}(f).
    \end{IEEEeqnarray}
    Both above inequalities exactly allow us to deduce that $\mathcal{E}_N \overset{M}{\rightarrow} \mathcal{E}$. 

\subsection{Generator Convergence}
Now that we have proven that $\mathcal{E}_N \overset{M}{\to}\mathcal{E}$, we may apply classical results of \cite{kuwaeshioya2003convergence} to state that the generators $\mathsf{L} \to \Delta_\mathcal{M}$ in the strong resolvent sense. We recall that $\mathsf L$ is understood as in Eq. \eqref{eq:smoothed-rwgl}, which acts intrinsically on $\mathcal{M}$. The lifting to extrinsic space will be done in the last step of the proof and all arguments below are intrinsic until specified otherwise. 

By standard consequences of strong resolvent convergence for self-adjoint operators, the associated Markov semigroups
\begin{equation}
T_t^{(N)} := e^{t\mathsf{L}},
\qquad
T_t := e^{t\Delta_{\mathcal M}}
\end{equation}
satisfy
\begin{equation}
T_t^{(N)} f \to T_t f
\qquad \text{strongly in } H
\end{equation}
for every \(f \in H\) and each fixed \(t \geq 0\).
\subsection{Weak Convergence of Process Law}
Let \(0 \leq t_1 < \cdots < t_k \leq T\), and let \(f_1,\dots,f_k \in C_b(\mathcal M)\). Denote by \((\mathsf{Z}_t)_t\) the Markov process generated by \(\mathsf L\), and by \((Z_t)_t\) the manifold diffusion generated by \(\Delta_{\mathcal M}\). By the Markov property,
\begin{equation}
\mathbb E_z\!\left[\prod_{j=1}^k f_j\!\left(\mathsf Z_{t_j}\right)\right]
=
T_{t_1}^{(N)}
\Bigl(
f_1 \,
T_{t_2-t_1}^{(N)}
\bigl(
f_2 \cdots
T_{t_k-t_{k-1}}^{(N)} f_k
\bigr)
\Bigr)(z),
\end{equation}
and similarly,
\[
\mathbb E_z\!\left[\prod_{j=1}^k f_j\!\left(X_{t_j}\right)\right]
=
T_{t_1}
\Bigl(
f_1 \,
T_{t_2-t_1}
\bigl(
f_2 \cdots
T_{t_k-t_{k-1}} f_k
\bigr)
\Bigr)(z).
\]
Since \(T_t^{(N)} \to T_t\) strongly for each fixed \(t\), an iterative application of the semigroup convergence implies
\[
\mathbb E_z\!\left[\prod_{j=1}^k f_j\!\left(Z^{(N)}_{t_j}\right)\right]
\to
\mathbb E_z\!\left[\prod_{j=1}^k f_j\!\left(Z_{t_j}\right)\right].
\]
Hence the finite-dimensional distributions of \((\mathsf Z_t)_t\) converge to those of \((Z_t)_t\).

It remains to prove tightness of the laws of \((\mathsf Z_t)_t\) on \(C([0,T];\mathcal M)\). Since \(\mathcal M\) is compact, the coefficients of the limiting diffusion are smooth and therefore bounded. Under the standing assumptions on the graph approximation, the approximating processes inherit uniform moment bounds on increments. In particular, there exists a constant \(C>0\), independent of \(N\), such that for all \(0 \leq s \leq t \leq T\),
\[
\sup_N \mathbb E_z \Bigl[d_{\mathcal M}\!\left(Z_t^{(N)},Z_s^{(N)}\right)^4\Bigr]
\leq
C |t-s|^2.
\]
By the Kolmogorov-Chentsov theorem \citep{ethier2009markov}, the family of laws of \((\mathsf Z_t)_t\) is tight in \(C([0,T];\mathcal M)\) for all $N$.

Combining tightness with convergence of finite-dimensional distributions \citep{billingsley2013convergence}, we conclude that
\[
\mathsf Z_t \Rightarrow Z_t
\qquad \text{in } C([0,T];\mathcal M),
\]
proving intrinsic convergence in law on path space. 

To prove the claim, we define the extrinsic processes as 
\begin{equation}
    \mathsf X_t := F(\mathsf Z_t), \quad Y_t := F(Z_t)
\end{equation}
under the embedding map $F:\mathcal{M}\rightarrow\mathbb{R}^n$. By the continuous mapping theorem, these processes inherit smoothness, whereby
\begin{equation}
    \mathsf X_t \implies Y_t \qquad \text{in } C([0,T];F(\mathcal M)),
\end{equation}
proving the claim.

\begin{flushright}
    \qed
\end{flushright}
\section{Proof of Theorem \ref{thrm:sampling-guarantees}}\label{apdx:sampling-proof}
\paragraph{Proof Outline}
We compare the Euler--Maruyama scheme \(\bar X_{\ell h}\) to the target smooth extrinsic diffusion \(X_t\) through the intermediate continuous data-driven process \(\hat X_t\). The total error is decomposed into (i) a graph approximation error between \(X_t\) and \(\hat X_t\), (ii) an off-grid time discretization error between \(\hat X_t\) and \(\hat X_{\ell h}\), and (iii) the Euler--Maruyama discretization error between \(\hat X_{\ell h}\) and \(\bar X_{\ell h}\).

\subsection{Preliminaries}
We begin by discussing the regularity conditions we assume. Since we are in a computational setting, we only consider all stochastic processes $(X_t)_{t\geq0}$ before their explosion time $e(X)$. A unique solution can be ensured to exist by assuming local Lipschitzness of the drift and diffusion terms \citep{hsu2002stochastic}, which is the setting we consider, formalized in the following Assumption.
\begin{assumption}[Uniform second-moment control]
\label{ass:moment} 
We assume that $\mathsf{L}$ and $\mathsf{\Gamma}$ satisfy a local Lipschitz growth condition uniformly in $N$, i.e.,
there exists constants $C,K>0$, independent of $N$, such that for all $x\in\mathbb{R}^n$,
\begin{IEEEeqnarray}{rCl}
\|\mathsf{L}(x) - \mathsf{L}(y)\| + \|\mathsf \Gamma^{\frac{1}{2}}(x) - \mathsf \Gamma^{\frac{1}{2}}(y)\| &\leq& C(\|x-y\|), \\
\|\mathsf L(x) \|^2 + \|\mathsf \Gamma^{\frac{1}{2}}(x)\|^2 &\leq& K(1+\|x\|^2).
\end{IEEEeqnarray}
Consequently, for every finite time horizon $T>0$, the process \eqref{eq:data-driven-SDE} is well behaved:
\begin{equation}
\sup_{N}\sup_{t\in[0,T]} \mathbb{E}\|\hbX^{(N)}_t\|^2 < \infty,
\qquad
\sup_{t\in[0,T]} \mathbb{E}\|\hbX_t\|^2 < \infty.
\end{equation}
\end{assumption}
\subsection{Splitting Terms}
We now consider the Euler-Maruyama discretization of the true $\Tilde{L}$-diffusion from Eq. \eqref{eq:ambient-manifdold-SDE} with timestep $h>0$, given by the difference equation
\begin{equation}
    \Bar{X}_{(\ell+1)h} := \Bar{X}_{\ell h} + h L(\Bar{X}_{\ell h}) + \Gamma^{1/2}(\Bar{X}_{\ell h})\sqrt{h}\xi_{\ell+1},
\end{equation}
where $\xi_{\ell+1} \sim \mathcal{N}(0,\mathbf{I}_k)$. 
We recall the squared Wasserstein-2 distance \citep{villani2009optimal}, 
\begin{equation}
    W_2^2(\mu,\nu) := \inf_{\gamma \in \Gamma(\mu,\nu)}\mathbb{E}_{(x,y)}\left[\|x-y\|_2^2\right],
\end{equation}
where $\Gamma(\mu,\nu)$ denotes the couplings of $\mu$ and $\nu$. It is easy to see that for two random variables $X\sim \mu$, $Y\sim \nu$, we have that 
\begin{equation}\label{eq:W-2-upper-bound}
    W_2^2(\mu,\nu) \leq \mathbb{E}\left[\|X-Y\|_2^2\right]
\end{equation}
by the optimality of the coupling $\gamma$ in $W_2$. 

We first introduce the intermediate continuous-time off-grid terms $\hat{X}$, and apply a triangle inequality on $W_2\left(\mathcal{L}(\hbX_{\ell h}), \mathcal{L}(X_t)\right)$, giving us 
\begin{IEEEeqnarray}{rCl}\label{eq:error-triangle-ineq}
    W_2\left(\mathcal{L}(\hbX_{\ell h}), \mathcal{L}(Y_t)\right) & \leq & \underbrace{W_2\left(\mathcal{L}(\hbX_{\ell h}), \mathcal{L}(\hat{X}_{\ell h})\right)}_{\textnormal{data approximation}} + \underbrace{W_2\left(\mathcal{L}(\hat{X}_{\ell h}), \mathcal{L}(\hat{X}_t)\right)}_{\textnormal{time discretization}} + \underbrace{W_2\left(\mathcal{L}(\hat{X}_t), \mathcal{L}(Y_t)\right)}_{\textnormal{E-M error}}, \nonumber  \\
\end{IEEEeqnarray}
which we will treat individually using different tools. 

\subsubsection*{Data Approximation}
We begin with the first term labeled "data approximation", which may be asymptotically suppressed by Theorem \ref{thrm:mosco-conv}. Indeed, the strong operator convergence ensures us that the data driven and off-grid processes will converge as $N\to\infty$. 

%\subsubsection*{Off-grid Time Discretization}
We now use Eq. \eqref{eq:W-2-upper-bound} to upper bound the term by 
\begin{equation}
    W_2\left(\mathcal{L}(\hat{X}_{\ell h}), \mathcal{L}(\hat{X}_t)\right) \leq \left(\mathbb{E}\left[\|\hat{X}_{\ell h} - \hat{X}_t\|_2^2\right]\right)^{1/2},
\end{equation}
which is controlled under linear growth by \textbf{synchronous coupling}, i.e. both $\hat{X}_t$ and $\hat{X}_{\ell h}$ are driven by the same Brownian motion $B_\tau$. Without loss of generality, let $0<s < t$ and consider the difference
\begin{equation}
    \hat{X}_s - \hat{X}_t = \int_s^tb(\hat{X}_\tau)\dd\tau + \int_s^t\sigma(\hat{X}_\tau)\dd B_\tau.
\end{equation}
Plugging this into the second term of Eq. \eqref{eq:error-triangle-ineq}, we get 
\begin{IEEEeqnarray}{rCl}
    \left(\mathbb{E}\left[\|\hat X_{t} - \hat{X}_{\ell h}\|_2^2\right]\right)^{1/2} & \leq & \left( \mathbb{E}\left[ \| \int_s^tb(\hat{X}_\tau)\dd\tau \|_2^2 \right] \right)^{1/2} +\left( \mathbb{E}\left[ \| \int_s^t\sigma(\hat{X}_\tau)\dd B_\tau \|_2^2 \right] \right)^{1/2} \nonumber \\
    & \overset{a)}{\leq} & \left(\mathbb{E}\left[\sup_{\tau\in(s,t)}\|b(\hat{X}_\tau)\|_2^2 \right]^{1/2} + \mathbb{E}\left[ \sup_{\tau\in(s,t)}\|\sigma(\hat{X}_\tau) \|_2^2\right]^{1/2}\right)\sqrt{|t-s|} \nonumber \\
    & \leq & C_t \sqrt{|t-s|},
\end{IEEEeqnarray}
where we use the maximum principle in $a)$, and bound it with a constant $C_t$. Thus, for $s=\ell h$, we can state that 
\begin{equation}\label{eq:off-grid-error}
    W_2^2(\mathcal{L}(\hat{X}_{\ell h}), \mathcal{L}(\hat{X}_t)) \in \mathcal{O}(h).
\end{equation}
Thus, by choosing an appropriate step size $h\rightarrow0$ as $N\rightarrow\infty$, we obtain the desired convergence. 

%\subsubsection*{E-M Error}
Finally, considering the regularity of Assumption \ref{ass:moment}, classical results of Euler-Maruyama schemes \citep{chewi25logconcave, jacod2011discretization} yield the bound
\begin{equation}
     W_2\left(\mathcal{L}(\hat{X}_t), \mathcal{L}(Y_t)\right)  \leq \mathcal{O}(h^{1/2})
\end{equation}

where in $a)$ we use the bound from Eq. \eqref{eq:W-2-upper-bound} on the two first terms. Notice that said terms purely pertain to discretization and "off-grid" errors of $\hat{X}$, whereas the last term is purely continuous, but compares the transportation cost from the $L$-diffusion to the data-driven diffusion.

By taking $N\to\infty$, $\epsilon\to 0$ as per Eq. \eqref{eq:eps-decay} and $h\to 0$, we combine all previous upper bounds and plug into Eq. \eqref{eq:error-triangle-ineq} to obtain 
\begin{equation}
    W_2\left(\mathcal{L}(\hbX_{\ell h}), \mathcal{L}(X_t)\right) \in \mathcal{O}(h^{1/2}),
\end{equation}
thus proving the claim.

\begin{flushright}
    \qed
\end{flushright}
\section{Implementation Details}\label{apdx:implementation-details}

\subsection{Graph Construction} 
To obtain the random walk graph Laplacian and consequently the CDC, we construct for all synthetic datasets the symmetric unweighted adjacency matrix via a global  bandwidth parameter $\epsilon>0$ as described in Section \ref{sec:data-driven-SDE}. In the case of the MNIST dataset, we used a fully-connected Gaussian kernel
\begin{equation}\label{apx:gaussian-kernel}
    k_\epsilon(x_i,x_j) := \exp\left(\frac{\|x_i-x_j\|_2^2}{\epsilon^2}\right).
\end{equation}
This non-compact kernel is not covered by the current convergence proof, although prior work \cite{trillos2018errorestimatesspectralconvergence} indicates the current framework should appeal to such kernels. We thus disclaim here that the MNIST experiments should be interpreted as empirical evidence, rather than fully theorem-backed instantiations. We compute the random walk graph Laplacian in the same way in both cases, using Eq. \eqref{eq:RWGL-action}.

 The CDC operator is computed by considering its action on coordinate functions. At each timestep $\ell$ of Algorithm \ref{alg:IMD-with-drgd}, we extend $\rwgl_N$ to $\mathsf{L}$ as in Eq. \eqref{eq:smoothed-rwgl} by using a nearest-neighbour (NN) or $k$-NN lookup. For the $k$-NN lookup, the CDC evaluated a sample $\hbX_\ell$ as in Eq. \eqref{eq:data-driven-EM} is computed as a Gaussian-kernel-weighted average of the CDC estimates of the $k$ nearest neighbours, with a locally adaptive bandwidth set from the median squared neighbour distance.

 \paragraph{Bandwidth parameter $\epsilon$} For the synthetic data, we set the global neighbourhood radius $\epsilon$ to the median of the Euclidean distances from each sample $x_i$ to its $k$-th nearest neighbour in $\{x_j\}_{j=1}^N \setminus {x_i}$, with $k = \max\bigl(10, \lceil 4 \log N \rceil\bigr)$. For the MNIST latent representation dataset, we set the Gaussian bandwidth to $\epsilon = Q_{0.9}\bigl(\{d_k(x_i)\}_{i=1}^N\bigr)$, the $90\%$ quantile of the distances $d_k(x_i)$ from each sample to its $k$-th nearest neighbour, with $k = \max\bigl(20, \lceil 4 \log N \rceil\bigr)$.
\begin{remark}
    Note that the NN lookup is generally controllable as per Assumption \ref{ass:moment}, and $k$-NN lookup is a Lipschitz map. As we are directly considering the extrinsic setting of $F(\mathcal{M})$, all functions are already considered lifted.
\end{remark}
\begin{remark}
    We also do not use the Diffusion Maps algorithm \cite{coifman2006diffusion}, as the variable-bandwidth graph Laplacian construction does not fit into the framework presented in Appendix \ref{apdx:graph-laplacian-convergence}. We believe this to be an important future direction for rigorous IMD constructions, but leave it for future work.
\end{remark}

\subsection{Datasets}\label{apdx:ssec:datadets}
\paragraph{Synthetic Data}
For the experiments on the hyperspheres $\mathbb{S}^2$, $\mathbb{S}^7$, torus and Swiss roll, we generate standardized synthetic data that is uniformly distributed on the manifold. We report the specific sample sizes for each experiment in \ref{apdx:ssec:exp_details}.

\paragraph{MNIST dataset details}\label{ssec:MNIST-latent}
To create latent embeddings for the MNIST dataset, we train a simple autoencoder. The encoder consists of two convolutional layers with ReLU activations that downsample the $1 \times 28 \times 28$ image to a $64 \times 7 \times 7$ representation, followed by a linear projection to a $64$-dimensional latent space. The decoder mirrors the encoder. We train the autoencoder on a reconstruction loss for $20$ epochs with a fixed learning rate of $1 \times 10^{-3}$ using the Adam optimizer. With the pretrained encoder we create a latent representation dataset of MNIST. To ensure stable training of the score model, we normalize the latent representations to zero mean and unit standard deviation. We use the resulting dataset to train a score model as described above and to run the IMD in MNIST latent space. 
% a linear projection to a $64 \times 7 \times 7$ feature map, followed by two transposed convolutional layers upsampling to $1 \times 28 \times 28$, with ReLU activations and a final sigmoid. 

\subsection{Score Model}\label{ssec:more-on-SGM}
For the experiments involving the complementary DRGD step (depicted in Fig. \ref{fig:imd_score_decomposition}), we train a noise-conditional score network $s_\sigma(x)$ \citep{hyvarinen05scorematch} following the multi-scale noise formulation of \cite{songyang2019dsm, songyang2020improvedsm} using denoising score matching. We parametrize the score model for the synthetic data and for the latent representations of the MNIST dataset with a fully-connected residual architecture and train the model with a noise-prediction loss. For the noise conditioning, we embed the noise level via a Gaussian random Fourier feature layer \citep{song2021scorebasedgenerativemodelingstochastic} and inject it as a bias into the first linear layer. We use a geometric schedule of noise levels. For the synthetic data, we use $20$ noise levels ranging from $\sigma_{\min}= 0.005$ to $\sigma_{\max}= 1$. For the latent MNIST representations we use $50$ noise levels ranging from $\sigma_{\min}= 0.01$ to $\sigma_{\max}= 10$.

\paragraph{Training Details} We train a network with $3$ blocks of residual layers of hidden dimension $128$ and SiLU activation functions using Adam (fixed learning rate $2 \times 10^{-4}$) on synthetic and latent MNIST datasets, for $200$ epochs with batch size $128$ on the synthetic data, and for $300$ epochs with batch size $256$ on the MNIST latent representations.

%General Training details:
%\begin{itemize}
%    \item Toy Data- epochs: 200, lr: 2e-4, Optimizer: Adam, 3 residual blocks, %hidden dims: 128
%    \item MNIST- epochs 1000, lr: 2e-4, Optimizer: Adam, 
%\end{itemize}

\paragraph{Noise Scale Selection} The noise scale $\sigma$ for the pretrained score model in the DRGD step of Algorithm \ref{alg:IMD-with-drgd} is chosen to be small enough such that no numerical errors are incurred during score estimations (as in \citep{kharitenko2025landingscoreriemannianoptimization}). In particular, we chose $\sigma = 0.005$ for all synthetic tasks, and $\sigma = 0.01$ for the MNIST tasks.

\subsection{Experimental Details on the Euler-Maruyama Parameters}\label{apdx:ssec:exp_details}
We specify all hyperparameters used to conduct the experiments in Section \ref{sec:experiments}.

\paragraph{Spherical Brownian Motion} In the experiments in Section \ref{ssec:Langevin} involving Brownian motion, we use a sample size of $10,000$ for the $\mathbb{S}^2$ manifold. To account for dimensionality, we increase the sample size for $\mathbb{S}^7$ to $100,000$ samples. We generate trajectories of $30$ different starting points, simulating a time horizon of $T=3$. This results in trajectories of $2$k steps for $h=10^{-3}$, $20$k steps for $h=10^{-4}$ and $200$k steps for $h=10^{-5}$.

\paragraph{Langevin dynamics} The Langevin sampling on the $\mathbb{S}^7$ hypersphere in Section \ref{ssec:Langevin} was performed with a step size of $h=10^{-4}$ and a sample size of $100,000$. We choose $\kappa = 10$ and $\boldsymbol{\mu} = e_1$. We simulate $100$ trajectories for $30$ different starting points, each for $20$k steps, simulating a time horizon of $T=2$. We take the endpoint of each trajectory to compute the empirical distribution of the test statistic, quantifying {statistical fidelity}, which reflects the analytical distribution described in Eq. \ref{eq:vMF-stat-law}. 
For numerical stability, we perform IMD with a DRGD step (see \ref{ssec:IMD-DRGD-discussion} for details) using a score model we obtain as described in \ref{ssec:more-on-SGM}. 

\paragraph{Guided on-manifold exploration of image data} 
As described in Section~\ref{ssec:MNIST-exp}, we conduct all MNIST experiments in the latent space of the pretrained autoencoder. For the ease of exposition and to obtain conceptually meaningful results, we only include digits from the classes $\{1, 2, 7\}$. Our findings remain qualitatively unchanged for other subsets. As described in Section~\ref{ssec:MNIST-exp}, we perform Langevin sampling using a quadratic potential (\ref{eq:mnist-potential}) with $\lambda = 4$ around a sample $Z_\star$ from the training dataset. The trajectories were generated with step size $h=10^{-4}$ for a total of $100$k steps, simulating a time horizon of $T=10$.
\paragraph{Disclosure of Compute Resources} All experiments were conducted on the authors' M4 Apple Silicon MacBook with 16Gb of RAM, and M3 Apple Silicon MacBook Air with 24Gb of RAM.
\section{Additional Experimental Results}\label{apdx:additional-exps}
Here we present complementary experimental results pertaining to those of Section \ref{sec:experiments}. In particular, we provide more details for Tables \ref{tab:geom-fide} and \ref{tab:SR-metrics}, along with some interpretation of the results. 
\subsection{Comparison with Naive Methods}\label{apdx:comparison-apdx}
To further support the unique correctness of our method (using \emph{both} $\mathsf{L}$ and $\mathsf{\Gamma}$), we compare experimental results of running more ``naive'' IMDs for spherical Brownian motion. As shown in Table \ref{fig:methods-comparison} and Figure \ref{fig:ballballball}, the score model is not the only driver of geometric fidelity. In particular, we compare Algorithm \ref{alg:IMD-with-drgd} with two other methods, inspired by recent applications of diffusion geometry to flow matching \cite{bamberger2025carreduchampflow} and using score-based models as retractions for Riemannian optimization in the implicit manifold setting \cite{kharitenko2025landingscoreriemannianoptimization}. The first method solely relies on the CDC operator, resulting in dynamics driven by 
\begin{IEEEeqnarray}{rCl}\label{eq:CDC-dynamics}
    \hbX_{\ell+1} & =  \hbX_\ell  + \sqrt{h}\mathsf \Gamma^{\frac{1}{2}}(\hbX_\ell)\xi_{\ell+1},
\end{IEEEeqnarray}
 The second method performs regular Brownian motion in $\mathbb{R}^n$, and relies on the DRGD step
\begin{IEEEeqnarray}{rCl}\label{eq:DRGD-dynamics}
    \mathsf U_{\ell} & = &  \hbX_\ell  + \sqrt{h}\xi_{\ell+1}, \\
    \hbX_{\ell+1} & = &  \mathsf U_\ell + \sigma^2s_\sigma(\mathsf U_\ell).
\end{IEEEeqnarray}
\begin{table}[h!]
\centering
\begin{comment}
\begin{figure}[htbp]
\centering
% -------- Row 1: diffusion plots --------
\begin{minipage}{0.32\textwidth}
    \centering
    \includegraphics[width=\linewidth]{figs/comparison/IMD_diffusion_plot.pdf}
    \subcaption{IMD}
\end{minipage}
\hfill
\begin{minipage}{0.32\textwidth}
    \centering
    \includegraphics[width=\linewidth]{figs/comparison/CDC_diffusion_plot.pdf}
    \subcaption{CDC only}
\end{minipage}
\hfill
\begin{minipage}{0.32\textwidth}
    \centering
    \includegraphics[width=\linewidth]{figs/comparison/BM_score_diffusion_plot.pdf}
    \subcaption{$\mathbb{R}^3$ Noise + DRGD Retraction}
    \vspace{10pt}
\end{minipage}
\end{figure}
\end{comment}
% -------- Row 2: error table --------
\begin{tabular}{lcc}
\toprule
\textbf{Method} & \textbf{Average radial error} & \textbf{Max. radial error} \\
\midrule
\textbf{IMD}                                      & $\mathbf{0.009 \pm 0.004}$ & $\mathbf{0.021 \pm 0.007}$ \\
CDC only                    & $0.474 \pm 0.014$ & $0.866 \pm 0.026$ \\
$\mathbb{R}^3$ Noise + DRGD Retraction   & $1.705 \pm 0.983$ & $3.732 \pm 1.629$ \\
\bottomrule
\end{tabular}
\vspace{10pt}
\caption{
Radial errors over trajectories of $5000$ steps with step size $h=10^{-3}$. We report the time-averaged radial error and the maximum radial error along each
trajectory, both averaged over $50$ independent runs ($\pm$ one standard deviation). IMDs maintain low, non-cumulative errors, whereas CDC only \eqref{eq:CDC-dynamics} and $\mathbb{R}^3$ noise + DRGD \eqref{eq:DRGD-dynamics} accumulate errors.
}
\label{fig:methods-comparison}
\end{table}

Observing Table \ref{fig:methods-comparison}, we immediately notice two things: (i) the CDC only construction accumulates errors in the normal direction and (ii) the isotropic noise + DRGD is not constrained to the sphere's surface. We interpret the first behaviour as evidence that the CDC as a projector onto $T_x\mathcal{M}$ is correct, but will induce normal error accumulation unless the step size is minuscule. Secondly, we notice that the score model based correction is not "strong" enough on its own to counter the isotropy of $\mathbb{R}^3$ Brownian motion, resulting in failure to maintain geometric consistency. 

\subsection{Swiss Roll}\label{apdx:other-geoms}
\begin{comment}
    
\begin{wrapfigure}{r}{0.30\textwidth}
    \vspace{-20pt}
    \centering
        \includegraphics[scale=0.1]{figs/swiss_roll.pdf}
    \caption{BM on the Swiss Roll.}
    \label{fig:swiss-roll-BM}
    \vspace{-15pt}
\end{wrapfigure}
\end{comment}

To further test IMDs on implicit manifolds, we consider the Swiss roll dataset. Since Euclidean proximity does not reflect geodesic locality, unconstrained or projection-based dynamics may induce transitions across nearby folds that are in fact distant along the manifold. IMDs avoid this failure mode by constructing dynamics at the level of the generator, ensuring that motion remains intrinsically constrained to the manifold. The data set for the Swiss roll experiments consists of $10,000$ samples. We compute trajectories with step size $h=10^{-3}$ for $10,000$ steps from $50$ random starting points.

To quantify complementary aspects of the induced dynamics, we report the following metrics:
\begin{itemize}
    \item \emph{Average nearest-neighbour (NN) distance}: the average ambient distance to the closest data point (best seen in Figs.~\ref{subfig:SR-a}, \ref{subfig:SR-b}, \ref{subfig:SR-c}). Off-manifold deviations are reflected in larger NN distances.
    \item \emph{Maximal jump}: a measure of the locality of the diffusion. We project the ambient trajectories onto the ``unrolled'' latent sheet of the Swiss roll and compute the distance between consecutive steps. This quantity should remain small under faithful manifold dynamics, but becomes large whenever the trajectory jumps across folds. The qualitative differences between methods are most visible in Figs.~\ref{subfig:SR-d}, \ref{subfig:SR-e}, \ref{subfig:SR-f}.
    % \item \emph{Spread} is the covariance of the endpoints $\textnormal{Cov}(X_T)$, which reflects how much the diffusion process has explored the manifold.
    % \item \emph{Mean Squared Displacement} reflects the average latent displacement, analogously to the spread.
\end{itemize}

 IMDs preserve geometric fidelity and local behaviour while still exploring the manifold. The empirical results (reported in Table~\ref{tab:SR-metrics} and Fig.~\ref{fig:SR-comparison}) confirm that the operator-based construction remains consistent with theory across all metrics, reflecting that IMDs approximate the correct infinitesimal dynamics rather than relying on heuristic geometric corrections.

% \begin{table}[t]
% \centering
% \vspace{5pt}
%     \scalebox{1}{
%         \begin{tabular}{lcccc}
%         \toprule
%         Method &  Avg NN Dist $\downarrow$ & Max Jump $\downarrow$ & Spread $\uparrow$ & MSD $\uparrow$ \\
%         \midrule
%         Only DRGD &  $0.65 \pm 0.431$ & $6.294$ & $16.701$ & ${20.172 \pm 2.442}$ \\
%         CDC &  $4.297 \pm 2.637$ & $6.464$ & ${112.320}^\dagger$ & ${123.304 \pm 12.268}^\dagger$ \\
%         \textbf{IMD} &  $\mathbf{0.298 \pm 0.075}$ & $\mathbf{2.141}$ &  $\mathbf{78.564}$ & $\mathbf{78.816 \pm 8.739}$\\
%         \bottomrule
%         \end{tabular}
%     }
%     \vspace{10pt}
% \caption{Swiss roll evaluation for a latent sheet of size $14\times 20$. Higher is better for spread and MSD; lower is better for NN distance and max jump. Reported values are means and standard errors over 50 starting points, 100 paths of 10k steps each. {\footnotesize
% $^\dagger$Reported for completeness but not interpreted, as these values are implausible due to non-local behavior.}}
% \label{tab:SR-metrics}
% \vspace{-15pt}
% \end{table}

\begin{table}[h!]
\centering
\vspace{5pt}
    \scalebox{1}{
        \begin{tabular}{lcc}
        \toprule
        \textbf{Method} &  \textbf{Average NN Distance} $\downarrow$ & \textbf{Max Jump} $\downarrow$ \\
        \midrule
        \textbf{IMD} & $\mathbf{0.019 \pm 0.000}$ & $\mathbf{0.116}$ \\
        CDC only & $0.736 \pm 0.035$ & $0.420$ \\
       $\mathbb{R}^3$ Noise + DRGD & $0.073 \pm 0.000$ & $0.375$ \\
        \bottomrule
        \end{tabular}
    }
    \vspace{10pt}
\caption{Swiss roll evaluation for a latent sheet. Lower is better for both NN distance and max jump. Reported values are means and standard errors over $50$ starting points with a fixed stepsize $h=10^{-3}$ for $10,000$ steps each.}
\label{tab:SR-metrics}
\vspace{-15pt}
\end{table}

\begin{figure}[h!]
\centering

% -------- Row 1: diffusion plots --------
\begin{minipage}{0.32\textwidth}
    \centering
    \includegraphics[width=\linewidth]{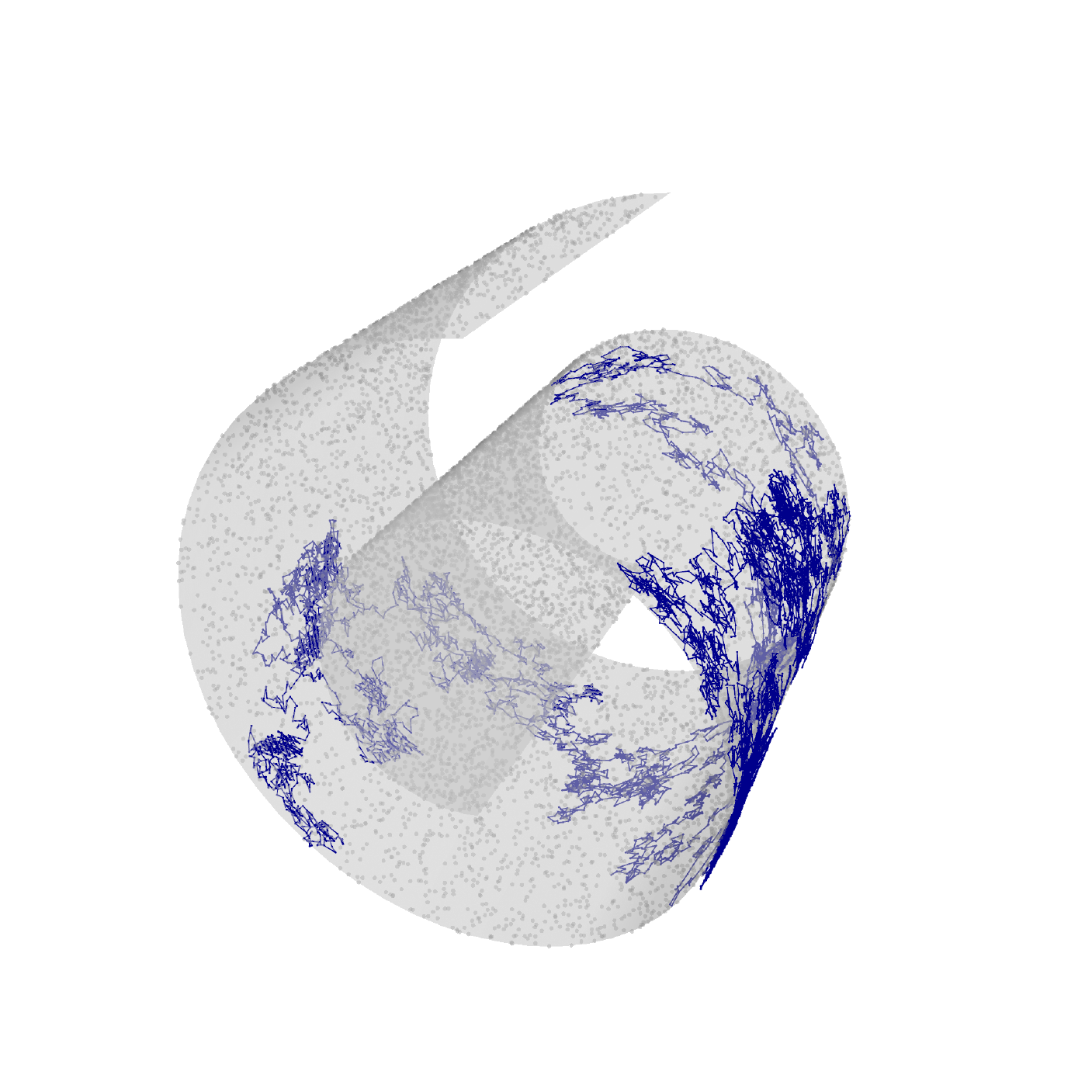}
    \subcaption{IMD}
    \label{subfig:SR-a}
\end{minipage}
\hfill
\begin{minipage}{0.32\textwidth}
    \centering
    \includegraphics[width=\linewidth]{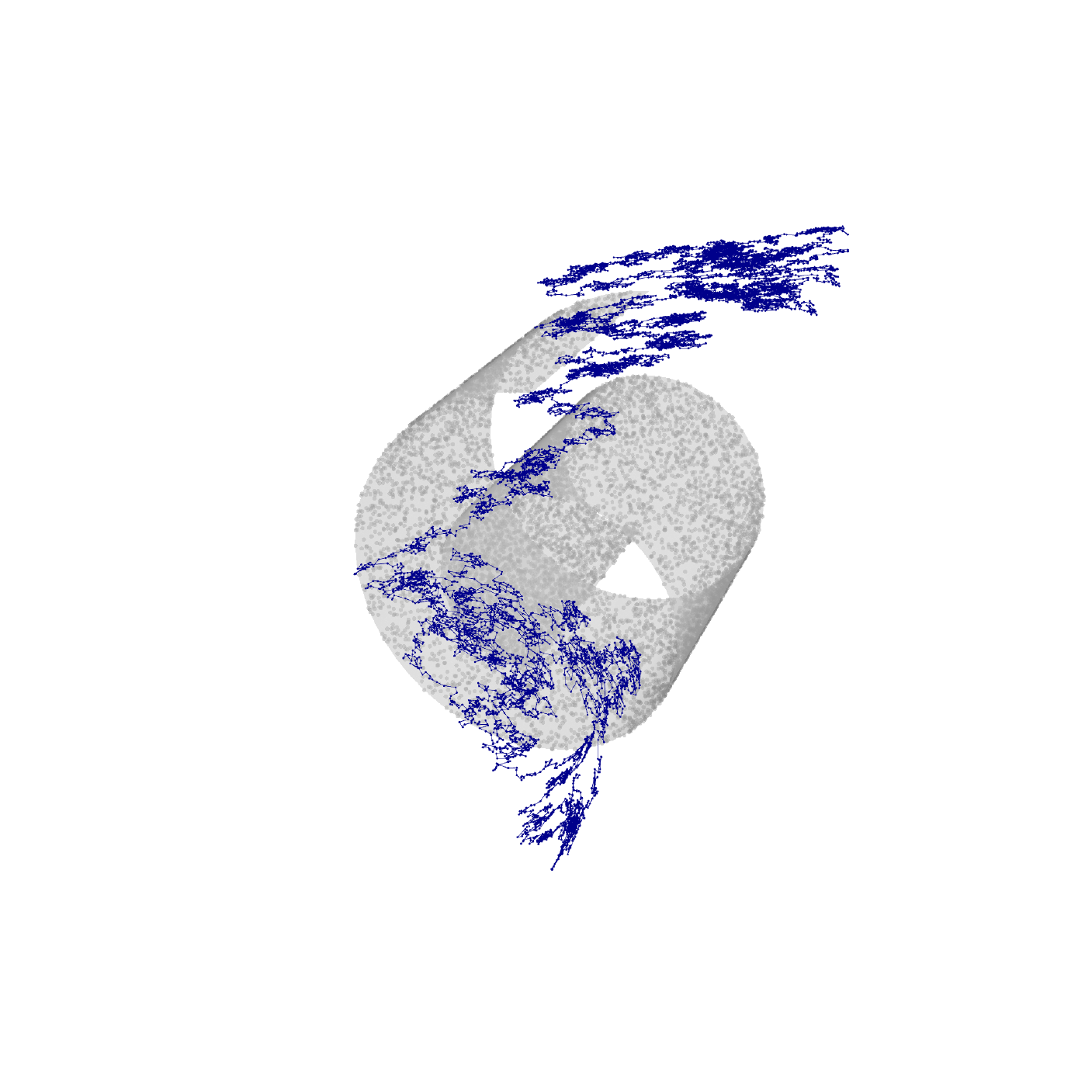}
    \subcaption{CDC only}
    \label{subfig:SR-b}
\end{minipage}
\hfill
\begin{minipage}{0.32\textwidth}
    \centering
    \includegraphics[width=\linewidth]{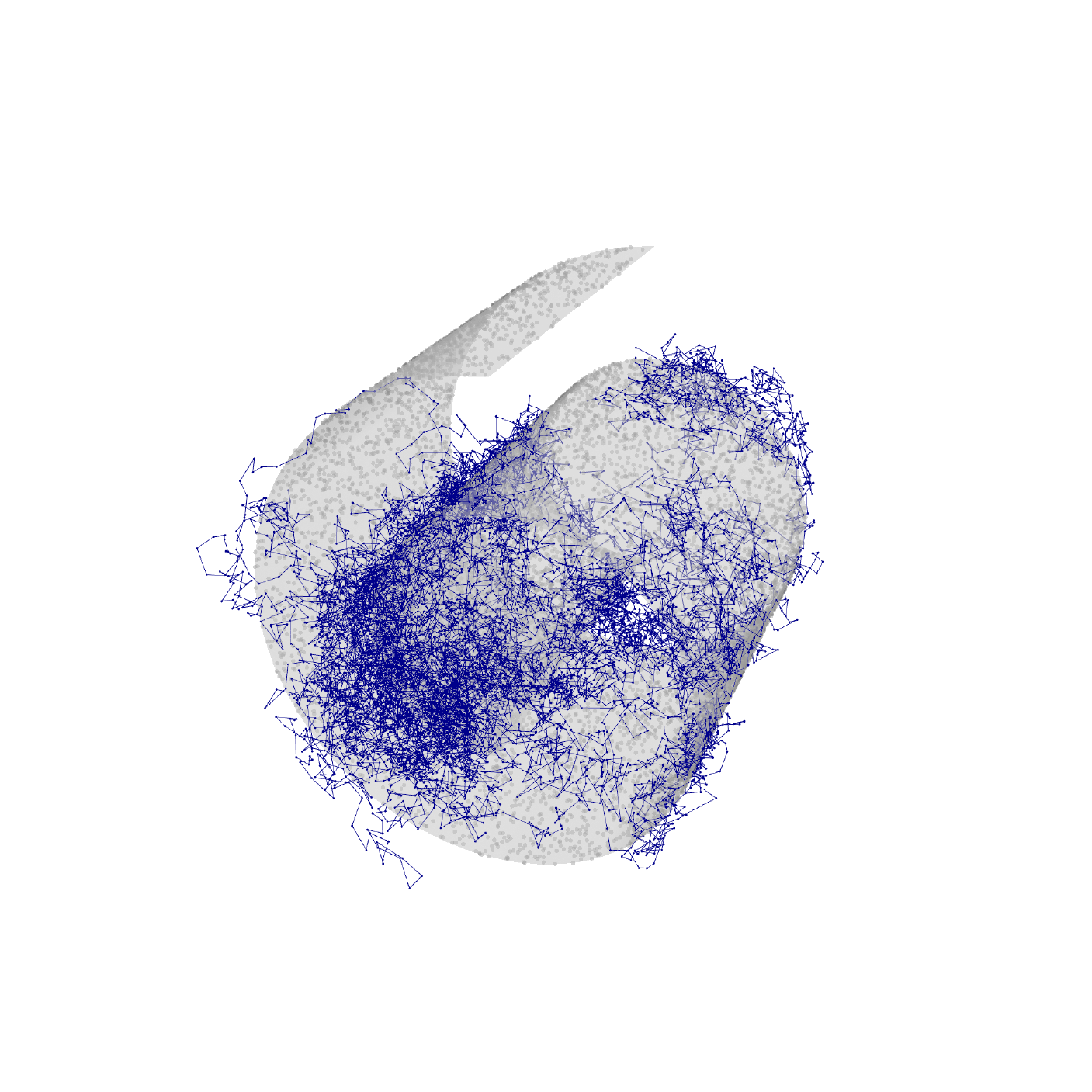}
    \subcaption{$\mathbb{R}^3$ Noise + DRGD Retraction}
    \label{subfig:SR-c}
\end{minipage}

\vspace{0.5em}

% -------- Row 2: error plots --------
\begin{minipage}{0.32\textwidth}
    \centering
    \includegraphics[width=\linewidth]{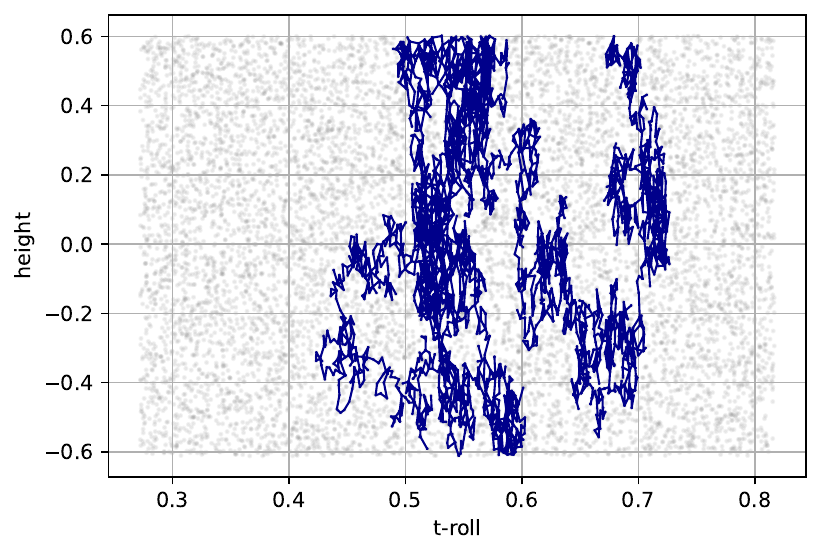}
    \subcaption{IMD on latent sheet}
    \label{subfig:SR-d}
\end{minipage}
\hfill
\begin{minipage}{0.32\textwidth}
    \centering
    \includegraphics[width=\linewidth]{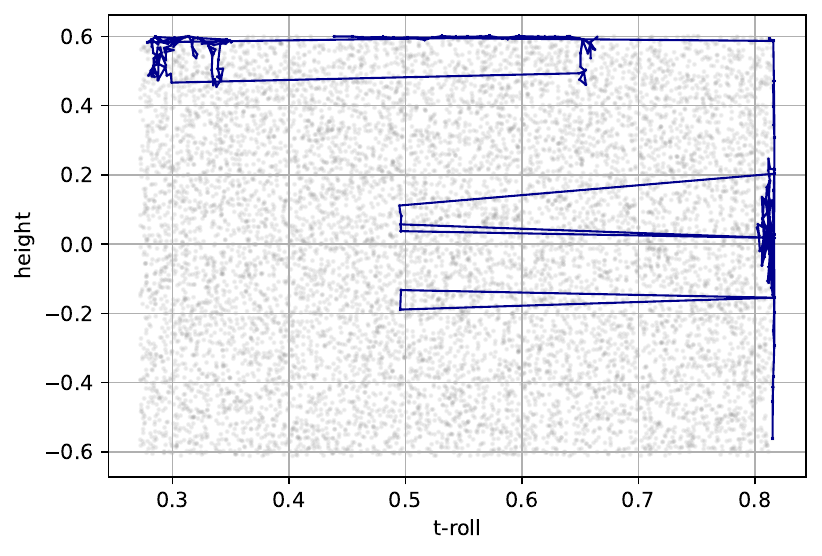}
    \subcaption{CDC on latent sheet}
    \label{subfig:SR-e}
\end{minipage}
\hfill
\begin{minipage}{0.32\textwidth}
    \centering
    \includegraphics[width=\linewidth]{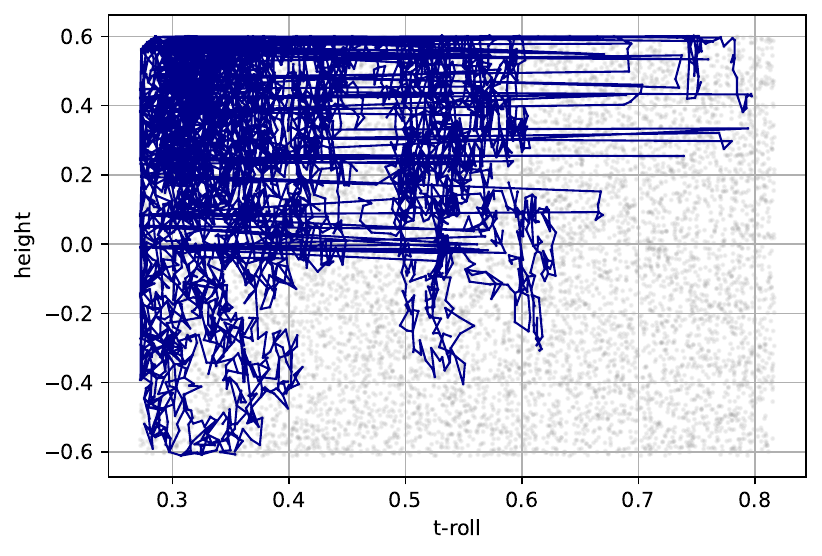}
    \subcaption{Naive noise on latent sheet}
    \label{subfig:SR-f}
\end{minipage}

\caption{
Comparison of diffusion trajectories ($h=10^{-3}$, $10$k steps) on the Swiss roll (top row) and corresponding latent sheet projections (bottom row) . Similar to Fig. \ref{fig:methods-comparison}, we notice error accumulation in the normal direction, which manifests non-locality as the discontinuous jump in the latent sheet. We also note the lack of continuous exploration of the isotropic BM compared to IMDs. 
}
\label{fig:SR-comparison}
\end{figure}

\subsection{Further Results on MNIST}
\begin{comment}
    
\begin{table}[h!]
    \centering
    \begin{tabular}{lcc}
			\toprule
			\multicolumn{1}{l}{} & $\mathbb{E}_\ell[\tau_\ell]$ & $\textnormal{Var}_\ell(\tau_\ell)$ \\ %& Decay $\hat{\lambda}$\\
			\midrule
			\multicolumn{1}{l}{DRGD} & $0.006$ & $0.0019$ \\ % & $-$ \\
			\multicolumn{1}{l}{IMD} & $\mathbf{0.0004}$ & $\mathbf{0.0003}$ \\ % & $-0.564\pm 0.057$ \\
			\bottomrule
		\end{tabular}
        \vspace{10pt}
    \caption{Local displacement corresponding to the MNIST trajectory of displayed in Fig. \ref{fig:MNIST-interpolation}.}
    \label{tab:MNIST-displacement}
\end{table}
We demonstrate ``Riemannian'' consistency of IMDs by evaluating the \emph{local squared displacement} 
\begin{equation}
    \tau_\ell := \delta_\ell^\top G(Z_\ell)\delta_\ell
\end{equation}
along a step $\delta_\ell := Z_{\ell+1}-Z_\ell$ of the latent path $(Z_\ell)$, where 
$G(\cdot) = J_D(\cdot)^\top J_D(\cdot)$
is the pullback metric induced by the Jacobian of the trained decoder $D:\mathcal{Z}\to\mathcal{X}$ \citep{shao2018riemannian}. This metric reflects how much the decoded image changes locally when moving from $Z_\ell$ to $Z_{\ell+1}$, as using $G$ rather than the Euclidean metric in latent coordinates better reflects the approximate Riemannian geometry of $\mathcal Z$. Table \ref{tab:MNIST-displacement} shows that DRGD exhibits a higher mean and standard deviation of these decoder-induced jumps,
indicating that its latent trajectory is less locally stable.
\end{comment}

\begin{figure}[h!]
    \centering

    \begin{minipage}[c]{0.5\textwidth}
        \centering
        \input{figs/bent-MNIST/bend-manifold}
        % or:
        % \includegraphics[width=\linewidth]{figures/my_figure.pdf}
    \end{minipage}
    \hfill
    \begin{minipage}[c]{0.48\textwidth}
        \caption{
        Implicit manifold interpolation in the MNIST latent space, guided by the potential \eqref{eq:mnist-potential}.
        IMDs (teal) follow the learned tangential geometry, while DRGD (red) rapidly moves toward the target through a score-based correction. Following Conjecture \ref{conj:curvature}, we interpret this behaviour to be consistent with locally flat intra-class geometry and lower-dimensional/curvature inter-class transitions.  
        }
        \label{fig:MNIST-bent-manifold}
    \end{minipage}

\vspace{-20pt}
\end{figure}
Figure \ref{fig:MNIST-bent-manifold} schematically further illustrates the off-manifold trajectory produced by DRGD: IMDs evolve along the tangential geometry of the learned data manifold, whereas DRGD acts as a normal score-based correction toward the target region.

\section{Support for Conjecture \ref{conj:curvature}}\label{apdx:conjecture-support}
The additional supporting evidence for Conjecture \ref{conj:curvature} is twofold. First, we will show how the second fundamental form $\Pi_x$ arises in the decomposition of the exponential map, and will relate it to current geometric integration schemes. Next, we will support our claim of stronger accumulating normal error as dimensions increase on the orthogonal matrix manifold $O(m)$. 

\subsection{Exponential Map Expansion}
The second fundamental form expansion from Eq. \eqref{eq:exp-decomp} explains how geometric Euler-Maruyama schemes \citep{zhan2026riemannianlangevindynamicsstrong, piggott2016geometric} differs from a naive extrinsic Euler-Maruyama step by a curvature-induced normal correction of order $h$. IMDs use an ambient Euler-Maruyama discretization of the data-driven generator, so they are consistent as $h\downarrow 0$, but at finite step size they may miss part of the random second-order normal geometry that a geometric Euler-Maruyama step or retraction would capture explicitly. This motivates investigating whether DRGD can act as a data-driven normal correction.

\paragraph{Spherical Error Decomposition} In the case of $\mathbb S^d_R$, $\Pi_x$ is analytically known, and the decomposition \eqref{eq:exp-decomp} yields
\begin{IEEEeqnarray}{rCl}
    \exp_x(\sqrt{h}\xi) & = &  x + \sqrt{h}\xi + \frac{h}{2}\Pi_x(\xi,\xi) + \mathcal{O}\left(\|\sqrt{h}\xi\|^3\right)\nonumber \\
    & = & x + \sqrt{h}\xi - \frac{h\|\xi\|^2}{2R^2}x+\mathcal{O}\left(\|\sqrt{h}\xi\|^3\right).
\end{IEEEeqnarray}
Since $\mathbb E[\|\xi\|^2]=d$, the average second-order normal correction is $-\frac{hd}{2R^2}x$, which exactly matches geometric Euler-Maruyama's mean curvature correction step. The per-step difference is thus 
\begin{equation}
    -\frac{h}{2R^2}\left(\|\xi\|^2-d\right)x,
\end{equation}
which has zero mean but fluctuates $\sim \frac{h}{R^2}\sqrt{\frac{d}{2}}$. We thus naturally expect normal error to accumulate as the intrinsic dimension $d$ increases, even by keeping the same step size.

\subsection{Orthogonal Matrix Group Experiment}\label{apdx:ssc-orthogonal-group}
We now further support the curvature argument from above on a different, less trivial manifold: the orthogonal matrix group $O(m)$, with an ambient dimension of $n=m^2$ (similar to the experiment in \citep{kharitenko2025landingscoreriemannianoptimization}), and intrinsic dimension $d =\frac{m(m-1)}{2}$. 

We chose $O(m)$ as an example, as the link to $\mathbb S^m$ is geometric: the sphere can be written as the homogeneous quotient
\begin{equation}
\mathbb S^{m-1}\simeq SO(m)/SO(m-1),
\end{equation}
where \(SO(m)\) acts by rotating directions and \(SO(d-1)\) is the stabilizer of one chosen direction. With the canonical invariant metric on \(SO(m)\), the projection
\begin{equation}
\pi:SO(m)\to \mathbb S^{m-1},\qquad \pi(Q)=Qe_m,
\end{equation}
is a Riemannian submersion, so the round metric and its constant positive curvature are inherited from the Lie-group geometry. Informally, the curvature of the sphere is encoded by the non-commutativity of infinitesimal rotations: brackets of horizontal directions in \(\mathfrak{so}(m)\) point into the stabilizer directions, producing the positive curvature of the quotient.

We evaluate IMD Brownian motion on a point cloud of $20,000$ samples for $m=3$; $50,000$ samples for $m=3$ and $100,000$ samples for $m=10$ from $O(m)$, and observe in Figure \ref{fig:stiefel-manifold} that the orthogonality error (measured as $\|\hbX_\ell\hbX_\ell^\top-I\|_F^2$, where $\hbX_\ell$ is the matrix produced by Eq. \eqref{eq:data-driven-EM}) increases as intrinsic dimension increases.

\begin{figure}[h]
    \centering
    \includegraphics[width=0.5\linewidth]{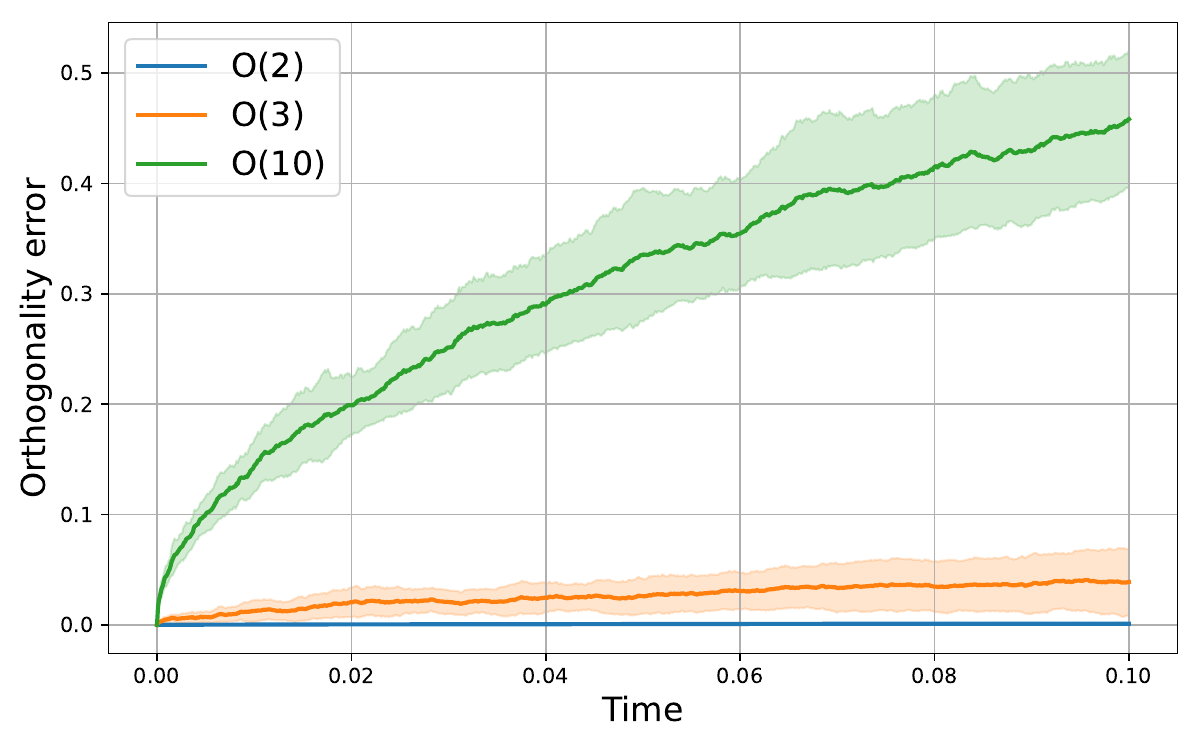}
    \caption{Orthogonality error of IMD Brownian motion trajectories on the orthogonal matrix manifold $O(m)$, for dimension $m \in \{2, 3, 10\}$ with a fixed step size of $10^{-4}$. Mean (solid line) and $\pm1$ standard deviation (shaded area) over 20 independent runs. Increasing error with increased dimensionality supports Conjecture \ref{conj:curvature}. }
    \label{fig:stiefel-manifold}
\end{figure}

\newpage
\end{document}